\documentclass[lettersize,journal]{IEEEtran}
\usepackage{amsmath,amssymb,amsfonts,amsthm}
\usepackage{algorithm}
\usepackage{algpseudocode}
\usepackage{amsmath}
\usepackage{array}
\usepackage[caption=false,font=footnotesize,labelfont=rm,textfont=rm]{subfig}
\usepackage{textcomp}
\usepackage{stfloats}
\usepackage{url}
\usepackage{verbatim}
\usepackage{graphicx}
\usepackage{cite}
\usepackage{xcolor}
\usepackage{hhline}
\usepackage{booktabs}
\usepackage{tabularx}
\usepackage{float}
\usepackage{amsmath}
\usepackage{multirow}
\usepackage{makecell}
\usepackage{cite}
\usepackage{colortbl}
\definecolor{lightgray}{gray}{0.95}

\newcolumntype{C}{>{\centering\arraybackslash}X}
\newcolumntype{L}{>{\raggedright\arraybackslash}X}
\newcolumntype{R}{>{\raggedleft\arraybackslash}X}

\theoremstyle{plain}
\newtheorem{theorem}{Theorem}
\newtheorem{proposition}{Proposition}
\newtheorem{lemma}{Lemma}
\newtheorem{corollary}{Corollary}
\theoremstyle{definition}
\newtheorem{definition}{Definition}
\theoremstyle{remark}
\newtheorem{remark}{Remark}

\usepackage{hyperref}
\newcommand{\Lobs}{L_{\mathrm{obs}}}
\newcommand{\Lmax}{L_{\mathrm{max}}}
\newcommand{\Lin}{L_{\mathrm{in}}}
\newcommand{\Lout}{L_{\mathrm{out}}}
\newcommand{\xobs}{x_{\mathrm{obs}}}
\newcommand{\Gam}{\mathrm{Gamma}}
\newcommand{\EE}{\mathbb{E}}
\newcommand{\RR}{\mathbb{R}}

\begin{document}
	
	\title{$\gamma$-Bridge: A Look-Parametric Diffusion Bridge}
	
	\author{Xuran Hu,
		Yujie Zhu, 
		Tengxi Wang, 
		Jilong Li, 
		and Wufan Zhao,~\IEEEmembership{Member,~IEEE}
		\thanks{Xuran Hu, Tengxi Wang, and Wufan Zhao are with the Urban Governance and Design Thrust, Society Hub, The Hong Kong University of Science and Technology (Guangzhou), Guangzhou, China (e-mail: XuRanHu@stu.xidian.edu.cn; twang744@connect.hkust-gz.edu.cn; wufanzhao@hkust-gz.edu.cn).}
		\thanks{Yujie Zhu is with Faculty of Science and Engineering, Macquarie University, NSW 2109, Sydney, Australia. (email:  yujie.zhu2@students.mq.edu.au).}
		\thanks{Jilong Li is with School of Geography and Planning, Ningxia University, Yinchuan, China (email: lijilong@nxu.edu.cn)}
		\thanks{Corresponding author: Wufan Zhao}}
	
	\IEEEpubid{}
	
	\maketitle

	\begin{abstract}
		Multiplicative Gamma noise is a signal-dependent degradation in coherent imaging; synthetic aperture radar (SAR) despeckling is its most prominent real-world instance. Existing diffusion denoisers parameterize their forward process by abstract signal-to-noise schedules rather than by the physical look number $L$, so different deployment scenarios typically require separately trained models, and transfer from synthetic Gamma training to real SAR remains challenging without clean ground truth. We introduce $\gamma$-Bridge, a look-parametric bridge whose schedule $L(t)$ connects the noisy observation at $\Lobs$ to the clean limit through exact multiplicative Gamma marginals. Its closed-form Gamma--L\'evy reverse posterior admits both stochastic and deterministic processes, while observation conditioning and a two-step consistency loss stabilize multi-step inference in the low-SNR single-look regime. Because bridge time directly represents $L$, one conditioned network can smart-start from any admissible input look and stop at a target look number. These two orthogonal controls enable zero-shot restoration over the full admissible grid after training only at $\Lobs=1$ on natural images with synthetic Gamma corruption. Combined with a homogeneous-patch look estimator, $\gamma$-Bridge processes data from six spaceborne and airborne SAR sensors without sensor-specific fine-tuning, achieving leading results on standard synthetic benchmarks while providing physically interpretable input and output controls absent from prior denoisers. Codes are released \href{https://github.com/Teriri1999/GammaBridge}{here}.
	\end{abstract}
	
	
	\section{Introduction}
	
	\IEEEPARstart{M}{ultiplicative} Gamma noise is a signal-dependent degradation encountered in coherent imaging systems, where it can obscure fine structures and radiometric information. Unlike additive Gaussian noise, its multiplicative and signal-dependent nature makes restoration a distinct image-processing problem. Synthetic aperture radar (SAR) despeckling is a major real-world instance and our primary validation setting. A SAR intensity image is commonly modeled as $Y=X\cdot N$, where $X$ denotes the underlying radar reflectivity and $N$ is unit-mean Gamma noise with variance $1/L$. The equivalent number of looks $L$ determines the noise strength; $L=1$ corresponds to intensity formed from single-look complex (SLC) data at approximately $0$~dB per-pixel SNR, while multi-look averaging or repeat-pass compositing increases $L$, producing cleaner observations at the cost of spatial or temporal resolution.
	
	Gamma image restoration has been studied using both model-based and learning-based approaches. Classical log-domain optimization and fractional-diffusion methods suppress Gamma noise while preserving image structures~\cite{Jose2010,DTFAD,GAO20241}. Diffusion-based methods include SAR-DDPM~\cite{10109106}, SDDPM~\cite{guha2023sddpm}, and S$^4$DM~\cite{heo2026self}, which employ Gaussian or transformed-domain diffusion processes. DDGM~\cite{nachmani2021denoising} introduces additive Gamma noise, while \cite{xie2023diffusion} formulate Gamma denoising through Beta-thinning. Recent studies unwind speckle with pretrained Gaussian denoisers~\cite{Lu2025DiSpeckle}, or offer adjustable diffusion-equation smoothing~\cite{Ran2026Tunable}. However, their diffusion coordinates or control variables are not explicitly exposed as an inference-time axis parameterized by the physical look number. Existing methods generally assume a fixed restoration endpoint or adjust generic smoothing strength, and therefore cannot independently specify the input and output Gamma noise levels through $(\Lin,\Lout)$ at inference time.
	\IEEEpubidadjcol
	A further difficulty is that real SAR scenes lack paired speckled and speckle-free ground truth. Supervised learning with synthetic Gamma corruption provides exact targets~\cite{sarcnn,idcnn}, but may suffer from a gap between idealized noise and sensor- or processing-dependent real SAR statistics~\cite{Bo2025SDUD}. Semi- and self-supervised methods reduce the need for clean references, yet representative approaches rely on multi-temporal observations, complex-valued measurements, or blind-spot independence assumptions~\cite{sar2sar,molini2021speckle2void}. Per-image zero-shot methods avoid external training data but require image-specific optimization and still return a fixed restoration endpoint~\cite{Ulyanov_2018_CVPR,albisani2025self}. Thus, label-free inference alone does not establish zero-shot transfer from controlled Gamma training to heterogeneous real SAR sensors without fine-tuning. Since clean real-SAR references remain unavailable, such transfer must also be assessed through distributional and no-reference evidence.
	
	This work introduces $\gamma$-Bridge, a look-parametric bridge for general Gamma image restoration, with SAR as its primary real-world application and validation setting. Building on the Gamma-marginal construction of~\cite{xie2023diffusion} and image-to-image diffusion bridges~\cite{liu2023}, this work defines a look-number schedule satisfying $L(0)=\Lmax$ and $L(T{-}1)=\Lobs$. Each intermediate state follows
	$x_t\mid x_0\sim x_0\cdot\Gam(L(t),L(t))$
	and therefore remains consistent with the standard $L(t)$-look Gamma observation model. The bridge time has a direct physical interpretation, and a single $L$-conditioned network learns the complete range of noise levels. Unlike methods with a fixed endpoint, $\gamma$-Bridge independently controls $\Lin$ and $\Lout$ at inference. For real deployment, an estimated effective look number anchors smart-start at the matching bridge state, enabling zero-shot application across multiple SAR sensors without speckle-free references or sensor-specific fine-tuning.
	
	The main contributions are summarized as follows.
	
	\begin{itemize}
		
		\item A look-parametric diffusion bridge with exact Gamma marginals. We formulate multi-look Gamma restoration as a bridge whose time axis is the physical look number $L$ rather than an abstract noise schedule. The Gamma--L\'evy decomposition yields a closed-form reverse posterior in both stochastic and deterministic forms.
		
		\item Two orthogonal look-number controls at inference. Target-$L$ output control terminates the reverse chain at any $\Lout$, and smart-start input control anchors an observation with look number $\Lin$ to the matching bridge state. A single model trained at $\Lobs=1$ thus performs zero-shot restoration over the full admissible $(\Lin,\Lout)$ grid. Observation conditioning and a two-step consistency loss further stabilize multi-step inference in the low-SNR single-look regime.
		
		\item Zero-shot transfer to real SAR with distributional validation. Trained only on natural images with synthetic Gamma corruption, $\gamma$-Bridge combines a homogeneous-patch $\hat L$ estimator with smart-start to process data from six spaceborne and airborne SAR sensors without sensor-specific fine-tuning, achieving leading results on standard synthetic benchmarks. Beta-coupled ratio statistics further validate the fidelity of the intermediate Gamma marginals at every reverse step.
		
	\end{itemize}
	
	\section{Related Work}
	\label{sec:related}

	\subsection{Diffusion Models and Non-Gaussian Gamma Processes}
	\label{sec:related-gamma-diff}
	SAR-DDPM~\cite{10109106} uses conditional Gaussian diffusion; SDDPM~\cite{guha2023sddpm} derives Gaussian forward and reverse processes under multiplicative corruption; S$^4$DM~\cite{heo2026self} performs self-supervised score estimation after a log-domain transformation; and DA-PhysDiff~\cite{Cai2026DAPhysDiff} introduces a physics-informed LogGamma stochastic differential equation. None of these formulations exposes $(\Lin,\Lout)$ as independently settable physical endpoint parameters.
	
	Non-Gaussian diffusion can better reflect image-formation statistics. DDGM~\cite{nachmani2021denoising} adds centered Gamma residual noise to a generative diffusion rather than modeling multiplicative speckle directly, while star-shaped DDPMs~\cite{Okhotin2023SSDDPM} accommodate exponential-family variables, including Beta and Wishart distributions. Xie \emph{et al.}~\cite{xie2023diffusion} treat Gaussian, Gamma, and Poisson degradations; their Gamma chain uses shape parameters $\alpha_0{=}\infty>\alpha_1>\cdots>\alpha_N$ and the Beta-thinning transition $x_{t+1}=(\alpha_t/\alpha_{t+1})\zeta_{t+1}x_t$, with $\zeta_{t+1}\sim\operatorname{Beta}(\alpha_{t+1},\alpha_t-\alpha_{t+1})$. This construction parallels the independent-Gamma ratio identity underlying our Gamma--L\'evy decomposition.
	
	The key distinction from Xie \emph{et al.} is twofold. First, we replace the abstract shape parameter $\alpha_t$ with the physical look number $L(t)$, so bridge time acquires a direct radar-imaging interpretation and becomes an admissible inference-time control axis. Second, we cast the process as a bridge terminating at the observation $x_{\text{obs}}$ rather than a generative chain terminating at a fixed prior, which admits smart-start input control at any $\Lin\in[\Lobs,\Lmax]$. Observation conditioning and two-step consistency regularize reversal at $\Lobs=1$, and Beta-coupled ratio statistics assess the fidelity of intermediate Gamma marginals.
	
	\subsection{Diffusion Bridges and Trajectory Consistency}
	I$^2$SB~\cite{liu2023} uses degraded images as informative endpoints; Denoising Diffusion Bridge Models~\cite{ICLR2024_20e45668} learn endpoint-conditioned score transport; and UNSB~\cite{ICLR2024_54912807} handles unpaired translation. DBIM~\cite{Zheng2025DBIM} accelerates bridge sampling, whereas RDBM~\cite{wang2026residual} spatially adapts restoration through residual-modulated perturbations. All of these formulations rely on Gaussian or Brownian perturbations, so their forward marginals are not physically meaningful outside the target endpoint; $\gamma$-Bridge instead preserves exact Gamma marginals at every step, and every intermediate state is a physically valid $L(t)$-look image.
	
	Consistency Models target one- or few-step generation, Consistency Trajectory Models~\cite{Kim2024CTM} map arbitrary times along a probability-flow ODE, and GCTMs~\cite{Kim2025GCTM} extend such maps to arbitrary distribution pairs and image restoration. Our two-step consistency plays a different role: it is a trajectory-level regularizer that provides a training signal $\mathcal{L}_{\text{rec}}$ cannot---at the pointwise $\ell_1$ optimum, $\nabla_\theta\mathcal{L}_{\text{rec}}=0$ while $\mathcal{L}_{\text{cons}}$ retains a non-trivial gradient along the reverse chain.

	\subsection{Model-Based and Learning-Based SAR Despeckling}
	Classical SAR despeckling includes local statistical filters~\cite{Lee1980,Frost1982}, nonlocal methods based on probabilistic patch similarity~\cite{Deledalle2009PPB,Parrilli2012,Deledalle2015}, and variational approaches using constrained optimization or variance stabilization~\cite{Woo2012Speckle,Mulog2017}.
	
	Learning-based methods include supervised networks trained with synthetic speckle~\cite{sarcnn,idcnn} and semi- or self-supervised methods that avoid clean SAR references~\cite{sar2sar,dalsasso2021if,molini2021speckle2void}. More recent approaches learn real-speckle distributions from unpaired data~\cite{Bo2025SDUD} or combine noise estimation with Transformer-based refinement~\cite{9884596}. In all these methods the training objective fixes a single target noise level, and the physical look number $L$ never appears as an inference-time variable; consequently the input and output speckle levels cannot be adjusted independently at test time.

	\section{Method}
	\label{sec:method}
	
	
	\subsection{Preliminaries}
	\label{sec:method-prelim}
	We adopt the shape--rate parameterization of the Gamma distribution throughout: $G \sim \Gam(\alpha, \beta)$ means
	\begin{equation}
		p_G(g) \;=\; \frac{\beta^{\alpha}}{\Gamma(\alpha)}\,
		g^{\alpha-1}\, e^{-\beta g},\qquad g > 0,
		\label{eq:gamma-pdf}
	\end{equation}
	where $\Gamma(\cdot)$ denotes the Gamma function, and $\EE[G] = \alpha/\beta$, $\mathrm{Var}(G) = \alpha/\beta^{2}$. Two elementary properties underpin the Gamma--L\'evy decomposition in Section~\ref{sec:method-levy}.
	
	\begin{lemma}[Scaling]\label{lem:gamma-scale}
		For any $c>0$, if $G \sim \Gam(\alpha, \beta)$ then $c\,G \sim \Gam(\alpha, \beta/c)$.
	\end{lemma}
	
	\begin{lemma}[Additivity]\label{lem:gamma-add}
		If $G_1 \sim \Gam(\alpha_1, \beta)$ and $G_2 \sim \Gam(\alpha_2, \beta)$ are independent (same rate~$\beta$), then $G_1 + G_2 \sim \Gam(\alpha_1 + \alpha_2, \beta)$.
	\end{lemma}
	
	Lemma~\ref{lem:gamma-scale} follows from a change of variables; Lemma~\ref{lem:gamma-add} is a standard characteristic-function calculation, recalled in Appendix~\ref{app:gamma-add} for completeness. A convenient consequence is the unit-mean Gamma: if $G \sim \Gam(\alpha, 1)$, then $G/\alpha \sim \Gam(\alpha, \alpha)$ has $\EE[G/\alpha]=1$ and $\mathrm{Var}(G/\alpha)=1/\alpha$, so a single scalar $\alpha$ controls both the mean and the variance.
	
	\subsection{Multiplicative Gamma Observation Model}
	\label{sec:method-problem}
	\begin{definition}[$L$-look SAR observation]\label{def:speckle} Let $x_0 \in \RR_{+}^{H \times W}$ be the underlying reflectivity. The observation $\xobs \in \RR_{+}^{H\times W}$ is related to $x_0$ by the multiplicative model
		\begin{equation}
			\begin{aligned}
				\xobs \;&=\; x_0 \odot N, \\
				N_{i,j} \;&\stackrel{\mathrm{i.i.d.}}{\sim}\; \Gam(\Lobs,\, \Lobs),
			\end{aligned}
			\label{eq:speckle}
		\end{equation}
		where $\odot$ denotes element-wise product. By the unit-mean Gamma property of Section~\ref{sec:method-prelim}, each speckle sample satisfies $\EE[N_{i,j}] = 1$ and $\mathrm{Var}(N_{i,j}) = 1/\Lobs$, so $\Lobs$ controls the noise strength through the variance alone.
	\end{definition}
	We drop pixel indices below; all equations act element-wise. The task is to produce an estimator $\hat{x}_0(\xobs)$ approximates $x_0$.
	
	\subsection{Look-parametric forward bridge}
	\label{sec:method-forward}
	
	\begin{definition}[Look schedule and forward bridge]\label{def:forward} Let $L : \{0, \ldots, T{-}1\} \to [\Lobs, \Lmax]$ be a strictly monotone look schedule with $L(0){=}\Lmax$ and $L(T{-}1){=}\Lobs$. The forward bridge $q(x_t \mid x_0)$ is specified by its marginals:
		\begin{equation}
			\begin{aligned}
				x_t \;&=\; \frac{x_0}{L(t)}\, G_t, \\
				G_t \;&\sim\; \Gam\bigl(L(t),\, 1\bigr).
			\end{aligned}
			\label{eq:forward-def}
		\end{equation}
	\end{definition}
	
	\begin{proposition}[Forward marginal]\label{prop:forward-marg}
		Under Eq.~\ref{eq:forward-def},
		\[
		x_t \;\sim\; x_0 \cdot \Gam\bigl(L(t), L(t)\bigr),
		\]
		i.e.\ $x_t$ is a physically valid $L(t)$-look SAR image of $x_0$. The endpoint at $t{=}0$ concentrates at $x_0$ as $\Lmax \to \infty$ and the endpoint at $t{=}T{-}1$ matches Eq.~\ref{eq:speckle}.
	\end{proposition}
	
	\begin{proof}
		Lemma~\ref{lem:gamma-scale} with $c = x_0/L(t)$: if $G_t \sim \Gam(L(t), 1)$ then $x_t = (x_0/L(t))\,G_t \sim \Gam\bigl(L(t),\, L(t)/x_0\bigr)$, which is precisely the law of $x_0 \cdot G$ with $G \sim \Gam(L(t), L(t))$ (a unit-mean Gamma; apply Lemma~\ref{lem:gamma-scale} again with $c=x_0$). The endpoint limits follow from $\mathrm{Var}(x_t) = x_0^{2}/L(t) \to 0$ as $L(t)\to\infty$ and $L(T{-}1) = \Lobs$ recovering Eq.~\ref{eq:speckle}.
	\end{proof}
	
	\begin{remark}[Physical interpretation]\label{rem:physical} The intermediate states are not abstract latent variables: they are $L(t)$-look SAR images. Every step of the bridge has a direct physical interpretation, which enables both target-$L$ output control and smart-start input control at inference time.
	\end{remark}
	
	\subsection{Gamma--L\'evy Decomposition}
	\label{sec:method-levy}
	
	The next step couples two adjacent bridge states. Consider a reverse step from $t$ to $t{-}1$ (so $L(t{-}1) > L(t)$). Define the increment
	\begin{equation}
		\begin{aligned}
			&Y_{t \to t{-}1} \;\sim\; \Gam\bigl(L(t{-}1) - L(t),\, 1\bigr), \\
			&Y_{t \to t{-}1} \text{ independent of } G_t.
		\end{aligned}
		\label{eq:increment}
	\end{equation}
	
	\begin{proposition}[L\'evy decomposition]\label{prop:levy} Let $G_{t-1} \;:=\; G_t + Y_{t\to t-1}$. Then $G_{t-1} \sim \Gam(L(t{-}1),\, 1)$, and
		\begin{equation}
			\begin{aligned}
				x_{t-1} \;&=\; \alpha_t\, x_t + \frac{x_0}{L(t{-}1)}\, Y_{t\to t{-}1}, \\
				\alpha_t \;&:=\; \frac{L(t)}{L(t{-}1)}.
			\end{aligned}
			\label{eq:levy}
		\end{equation}
	\end{proposition}
	
	\begin{figure}[!t]
		\centering
		\includegraphics[width=\linewidth]{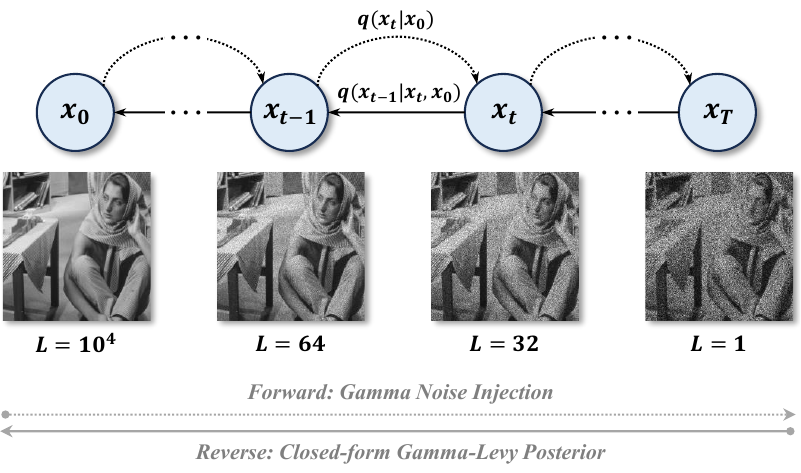}
		\caption{Overview of $\gamma$-Bridge: the bridge time indexes the physical look number $L(t)$, connecting the clean limit to the single-look observation through exact Gamma marginals.}
	\end{figure}
	
	\begin{proof}
		$Y_{t\to t{-}1} \sim \Gam(L(t{-}1){-}L(t), 1)$ and $G_t \sim \Gam(L(t), 1)$ share rate~$1$ and are independent by construction; Lemma~\ref{lem:gamma-add} gives $G_{t-1} \sim \Gam(L(t{-}1), 1)$. Substituting Eq.~\eqref{eq:forward-def} at step $t{-}1$,
		\begin{align*}
			x_{t-1} \;&=\; \frac{x_0}{L(t{-}1)}\,(G_t + Y_{t\to t{-}1}) \\
			&=\; \frac{L(t)}{L(t{-}1)} \cdot \frac{x_0}{L(t)}\,G_t
			+ \frac{x_0}{L(t{-}1)}\, Y_{t\to t{-}1} \\
			&=\; \alpha_t\, x_t + \frac{x_0}{L(t{-}1)}\, Y_{t\to t{-}1}. \qedhere
		\end{align*}
	\end{proof}
	
	
	\begin{remark}[Structural consequences]\label{rem:levy-consequences} Two immediate consequences of Prop.~\ref{prop:levy}: (i) the reverse step is a convex-like interpolation between the current state and a Gamma noise draw scaled by $x_0$; (ii) the bridge never adds negative mass, so $x_t > 0$ holds throughout without any clipping heuristic.
	\end{remark}
	
	\subsection{Closed-form reverse posterior}
	\label{sec:method-posterior}
	
	The reverse posterior $q(x_{t-1} \mid x_t, x_0)$ is what a sampler uses when $x_0$ is unknown and replaced by a network estimate $\hat{x}_0$. We derive two forms.
	
	\smallskip
	
	\noindent \textbf{Stochastic reverse.}
	Eq.~\ref{eq:levy} defines the reverse transition directly: given $x_t$ and $x_0$, draw an independent $Y_{t\to t{-}1} \sim \Gam(L(t{-}1){-}L(t), 1)$ and set
	\begin{equation}
			\begin{aligned}
				x_{t-1} \;&=\; \alpha_t\, x_t \;+\; \frac{\hat{x}_0}{L(t{-}1)}\, Y_{t\to t{-}1}, \\
				Y_{t\to t{-}1} \;&\sim\; \Gam\bigl(L(t{-}1){-}L(t),\, 1\bigr).
			\end{aligned}
		\label{eq:posterior-stoch}
	\end{equation}

	\noindent \textbf{Deterministic reverse.}
	The deterministic form replaces the Gamma draw by its expectation. $\EE[Y_{t\to t{-}1}] = L(t{-}1) - L(t)$, so
	\begin{equation}
		\begin{aligned}
			\EE_{Y_{t\to t{-}1}}\!\left[\frac{\hat{x}_0}{L(t{-}1)}\, Y_{t\to t{-}1}\right]
			\;&=\; \frac{L(t{-}1) - L(t)}{L(t{-}1)}\, \hat{x}_0 \\
			\;&=\; (1-\alpha_t)\,\hat{x}_0.
		\end{aligned}
	\end{equation}
	Taking expectations of Eq.~\ref{eq:posterior-stoch},
	\begin{equation}
			x_{t-1} \;=\; \alpha_t\, x_t \;+\; (1-\alpha_t)\,\hat{x}_0.
		\label{eq:posterior-otode}
	\end{equation}
	The deterministic form is an affine mixture between the current state and the network's clean estimate, with the mixture coefficient $\alpha_t = L(t)/L(t{-}1)$ set entirely by the schedule. Both posterior forms are self-consistent with the forward marginal of Prop.~\ref{prop:forward-marg} when the oracle $x_0$ is used.
	
	\begin{theorem}[Marginal preservation]\label{thm:marginal} Assume $\hat{x}_0 = x_0$. Then Eq.~\ref{eq:posterior-stoch} yields
		\[
		x_{t-1} \;\sim\; x_0 \cdot \Gam\bigl(L(t{-}1),\, L(t{-}1)\bigr),
		\]
		and Eq.~\ref{eq:posterior-otode} yields $\EE[x_{t-1}] = x_0$.
	\end{theorem}
	
	The proof is deferred to Appendix~\ref{app:posterior}. Theorem~\ref{thm:marginal} guarantees marginal-level correctness, but a diffusion sampler produces trajectories rather than isolated states; marginal matching at every $t$ is strictly weaker than joint-law matching. The next result upgrades Theorem~\ref{thm:marginal} to the trajectory level.
	
	\smallskip
	
	\noindent \textbf{The Gamma--L\'evy coupling.}
	Eq.~\ref{eq:forward-def} specifies the bridge only through its marginals, which do not uniquely determine a joint distribution across time. We introduce a specific coupling (call it Gamma--L\'evy coupling), induced by Proposition~\ref{prop:levy}: sample $G_{T-1}$ once, and generate every other state by adding independent L\'evy increments $Y_{t\to t{-}1}$ (independent of everything at times $\ge t$). Under this coupling the forward process is a Markov chain in the reverse-time index and admits the joint distribution
	\begin{equation}
		\begin{aligned}
			q(x_{T-1}, &\ldots, x_0 \mid x_0) \\
			&=\; q(x_{T-1} \mid x_0)\, \prod_{t=1}^{T-1} q(x_{t-1} \mid x_t, x_0),
		\end{aligned}
		\label{eq:joint-forward}
	\end{equation}
	whose kernels are given exactly by Eq.~\ref{eq:posterior-stoch}. Any other coupling that shares the same marginals but differs on higher-order dependencies would also be consistent with Eq.~\ref{eq:forward-def}; we work with the Gamma--L\'evy coupling because it is the one implicitly realised by the reverse sampler.
	
	\begin{remark}[]\label{rem:notation}
		The $x_0$ to the right of the conditioning bar denotes the clean reflectivity (a deterministic input); the $x_0$ inside the tuple denotes the terminal bridge state at $t{=}0$, a random variable with mean $x_0$ and variance $x_0^{2}/\Lmax$. The two coincide as $\Lmax \to \infty$; for our $\Lmax{=}10^{4}$ the residual coefficient of variation is $10^{-2}$, small enough to justify the notational overloading below.
	\end{remark}
	
	\begin{proposition}[Joint-law preservation under the oracle]
		\label{prop:joint}
		Suppose $\hat{x}_0 = x_0$ in Eq.~\ref{eq:posterior-stoch}, and let the reverse chain be initialised at $x_{T-1} = \xobs \sim q(x_{T-1}\mid x_0)$. Then the trajectory $(x_{T-1}, x_{T-2}, \ldots, x_0)$ generated by iterating Eq.~\ref{eq:posterior-stoch} has law equal to the forward joint $q(x_{T-1}, \ldots, x_0 \mid x_0)$ of Eq.~\ref{eq:joint-forward}.
	\end{proposition}
	
	\begin{proof}
		By Proposition~\ref{prop:levy}, the L\'evy-generated forward process is Markov in reverse-time index: $q(x_{t-1}\mid x_{t:T-1}, x_0) = q(x_{t-1}\mid x_t, x_0)$ because the increment $Y_{t\to t-1}$ is drawn independently at each descending step. The transition kernel is exactly Eq.~\ref{eq:posterior-stoch}. The sampled chain uses (a) the correct initial law $q(x_{T-1}\mid x_0)$ by assumption on the initial state and (b) the correct transition $q(x_{t-1}\mid x_t, x_0)$ at every step under $\hat{x}_0 = x_0$. Applying the chain rule $q(x_{T-1}, \ldots, x_0 \mid x_0) = q(x_{T-1}\mid x_0) \prod_t q(x_{t-1}\mid x_t, x_0)$ concludes.
	\end{proof}
	
	\begin{corollary}[Distributional correctness under the Gamma--L\'evy coupling]\label{cor:intermediate} Under the oracle $\hat x_0 {=} x_0$ and under the Gamma--L\'evy coupling of Eq.~\ref{eq:joint-forward}, any finite-dimensional statistic computed across bridge states---cross-step ratio images $\xobs/x_t$, ENL of partially denoised outputs, no-reference metrics at intermediate $L(t)$---has the correct distribution. The claim does not extend to couplings that share only the marginals of Eq.~\ref{eq:forward-def}, so it is Prop.~\ref{prop:joint} rather than Theorem~\ref{thm:marginal} that carries the joint-law guarantee.
	\end{corollary}
	
	\begin{figure}[!t]
		\centering
		\includegraphics[width=\linewidth]{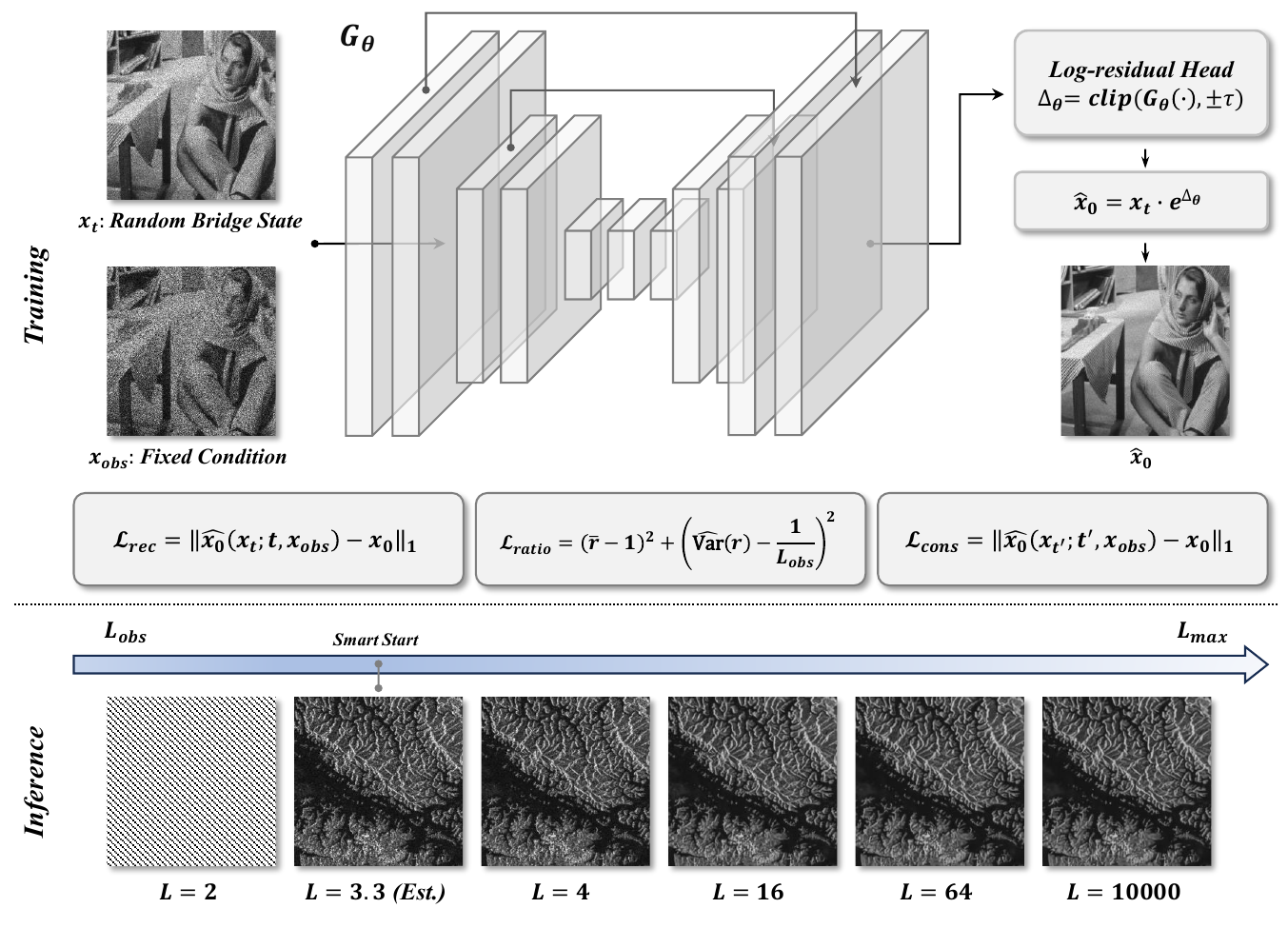}
		\caption{Training (top): $G_\theta$ maps a random bridge state $x_t$ and the observation $x_{\text{obs}}$ to $\hat{x}_0$ via a log-residual head, supervised by $\mathcal{L}_{\text{rec}}$, $\mathcal{L}_{\text{ratio}}$, and $\mathcal{L}_{\text{cons}}$. Inference (bottom): the estimated look number $\hat{L}$ triggers smart-start at the matching bridge step, and target-$L$ stopping produces outputs at any $L_{\text{out}}$.}
	\end{figure}
	
	%
	%
	
	\noindent\textbf{Relation to the network predictor.}
	In practice $\hat{x}_0$ is a network estimate rather than the true $x_0$, and the two levels of guarantee established above degrade differently.
	
	\textit{(i) Marginal correctness is retained pointwise.}
	If the network approximates the posterior mean, $\hat{x}_0(x_t, \xobs) \approx \EE[x_0 \mid x_t, \xobs]$, the deterministic reverse of Eq.~\eqref{eq:posterior-otode} tracks the correct posterior-mean trajectory at every step. We supervise the predictor with an $\ell_1$ reconstruction loss $\mathcal{L}_{\mathrm{rec}}$, whose pointwise minimizer is the conditional median. For the multiplicative-Gamma posterior, approximately unimodal and light-tailed after log transformation, mean and median coincide to within a few percent, so the trained predictor tracks the marginal mean to within $1\%$ across most of the schedule.
	
	\textit{(ii) Joint correctness is not automatic.}
	A predictor that is marginally unbiased at every $t$ can still produce trajectories with the wrong joint law: two predictions $\hat{x}_0(x_t, \xobs)$ and $\hat{x}_0(x_{t'}, \xobs)$ obtained at different steps of the same trajectory need not be consistent with each other, and their disagreement propagates into the joint distribution of $(x_{T-1}, \ldots, x_0)$. The consistency loss $\mathcal{L}_{\mathrm{cons}}$ (Eq.~\eqref{eq:loss-cons}) supplies a trajectory-level training signal aligned with Prop.~\ref{prop:joint}, reconstructing $x_0$ from the state reached by advancing the deterministic posterior from $t$ to $t'$.
	
	\begin{remark}[Scope of $\mathcal{L}_{\mathrm{cons}}$]\label{rem:cons-heuristic}
		$\mathcal{L}_{\mathrm{cons}}$ is defined along the deterministic trajectory (Eq.~\eqref{eq:cons-x-mid}), so it enforces trajectory consistency for the deterministic reverse of Eq.~\eqref{eq:posterior-otode}, whereas Prop.~\ref{prop:joint} concerns the stochastic chain Eq.~\eqref{eq:posterior-stoch}. The two are aligned in structure---both are joint-level statements about how predictions at different bridge times must agree---but deterministic-reverse consistency does not formally entail stochastic-chain joint-law fidelity. Empirically, combining $\mathcal{L}_{\mathrm{cons}}$ with $\mathcal{L}_{\mathrm{rec}}$ keeps the reverse trajectory on-manifold and improves multi-step inference in the single-look regime (Section~\ref{sec:exp-multistep}); a stochastic-chain consistency loss remains an open direction.
	\end{remark}
	
	\begin{remark}[Contrast with Gaussian bridges]\label{rem:gaussian-cf}
		Unlike I$^{2}$SB~\cite{liu2023}, which uses a Brownian-bridge SDE with Gaussian transition kernels, our posterior derives from the additivity of independent Gamma random variables (Lemma~\ref{lem:gamma-add}). Two structural consequences follow: (i) samples remain strictly positive by construction, so no clipping is needed to enforce the non-negativity of intensity images; (ii) the schedule enters only through the physical look number $L(t)$, giving bridge time an interpretation as an admissible inference-time control axis rather than an abstract signal-to-noise coordinate.
	\end{remark}
	
	\subsection{Log-residual parameterisation}
	\label{sec:method-param}
	
	Because the observation model is multiplicative in $x_0$, we parameterise the network to predict a multiplicative correction. Let $G_\theta$ be a UNet with inputs $(x_t,\, t,\, \log L(t),\, c)$, where $c=\xobs$ is the observation-conditioning channel introduced in Section~\ref{sec:method-train}. We define
	\begin{equation}
		\begin{aligned}
			\hat{x}_0(x_t; t, c) \;&=\; x_t \cdot \exp\!\bigl(\Delta_\theta\bigr), \\
			\Delta_\theta \;&:=\; \mathrm{clip}\!\bigl(G_\theta(x_t, t, \log L(t), c),\; -\tau,\, +\tau\bigr),
		\end{aligned}
		\label{eq:log-residual}
	\end{equation}
	where $\tau$ is a stability clip that bounds any single-step correction to a factor of $e^{\pm\tau}$; we set $\tau=5$ throughout, a bound never approached in practice on any converged checkpoint.
	
	\begin{remark}[Why log-residual]\label{rem:log-residual}
		Two motivations underpin Eq.~\eqref{eq:log-residual}. First, the true multiplicative correction has a clean logarithmic form: since $x_t = (x_0/L(t))\,G_t$ (Eq.~\eqref{eq:forward-def}), we have $\log(x_0/x_t) = -\log(G_t/L(t))$, so the network target $\Delta_\theta$ is directly aligned with the (unit-mean) Gamma noise realisation---a natural regression variable. Second, $\hat{x}_0 > 0$ is guaranteed for every $G_\theta$ output, automatically enforcing the non-negativity of intensity images without any constraint on the network's output range. 
		
	\end{remark}
	
	\subsection{Training objective}
	\label{sec:method-train}
	
	Training requires only a pool of clean natural images: an $x_0$-pool of clean images and a sampler that produces $x_t$ via Eq.~\ref{eq:forward-def} at random $t$. The conditioning input is $c = \xobs = x_0/\Lobs \cdot G_{T{-}1}$ drawn independently, i.e.\ the network always sees the observation at every reverse step (matching inference).
	
	The objective has three terms.
	
	\smallskip
	
	\noindent\textbf{Reconstruction.}
	For $t \sim \mathrm{Unif}\{1,\ldots,T{-}1\}$ and $x_t$ sampled from $q(x_t \mid x_0)$,
	\begin{equation}
		\mathcal{L}_{\mathrm{rec}}
		\;=\; \EE_{x_0, t, x_t}\,\bigl\|\hat{x}_0(x_t; t, \xobs) - x_0\bigr\|_1.
		\label{eq:loss-rec}
	\end{equation}
	
	\noindent\textbf{Radiometric moment matching.}
	The residual field $r = \xobs / \hat{x}_0$ should follow $\Gam(\Lobs, \Lobs)$, hence $\EE[r]=1$ and $\mathrm{Var}(r)=1/\Lobs$. We enforce these two moments through
	\begin{equation}
		\mathcal{L}_{\mathrm{ratio}}
		\;=\; (\overline{r} - 1)^{2}
		\;+\; \bigl(\widehat{\mathrm{Var}}(r) - \tfrac{1}{\Lobs}\bigr)^{2},
		\label{eq:loss-ratio}
	\end{equation}
	where $\overline{r}$ and $\widehat{\mathrm{Var}}(r)$ are the sample mean and variance pooled over the batch and spatial pixels. This pooled estimator implicitly treats $r_{i,j}$ as identically distributed across pixels, which is exact under the multiplicative Gamma model when $\hat{x}_0$ is uncorrelated with the noise realisation and holds approximately when the predictor introduces only mild spatial correlation.
	
	\begin{figure}[!t]
		\centering
		\includegraphics[width=\linewidth]{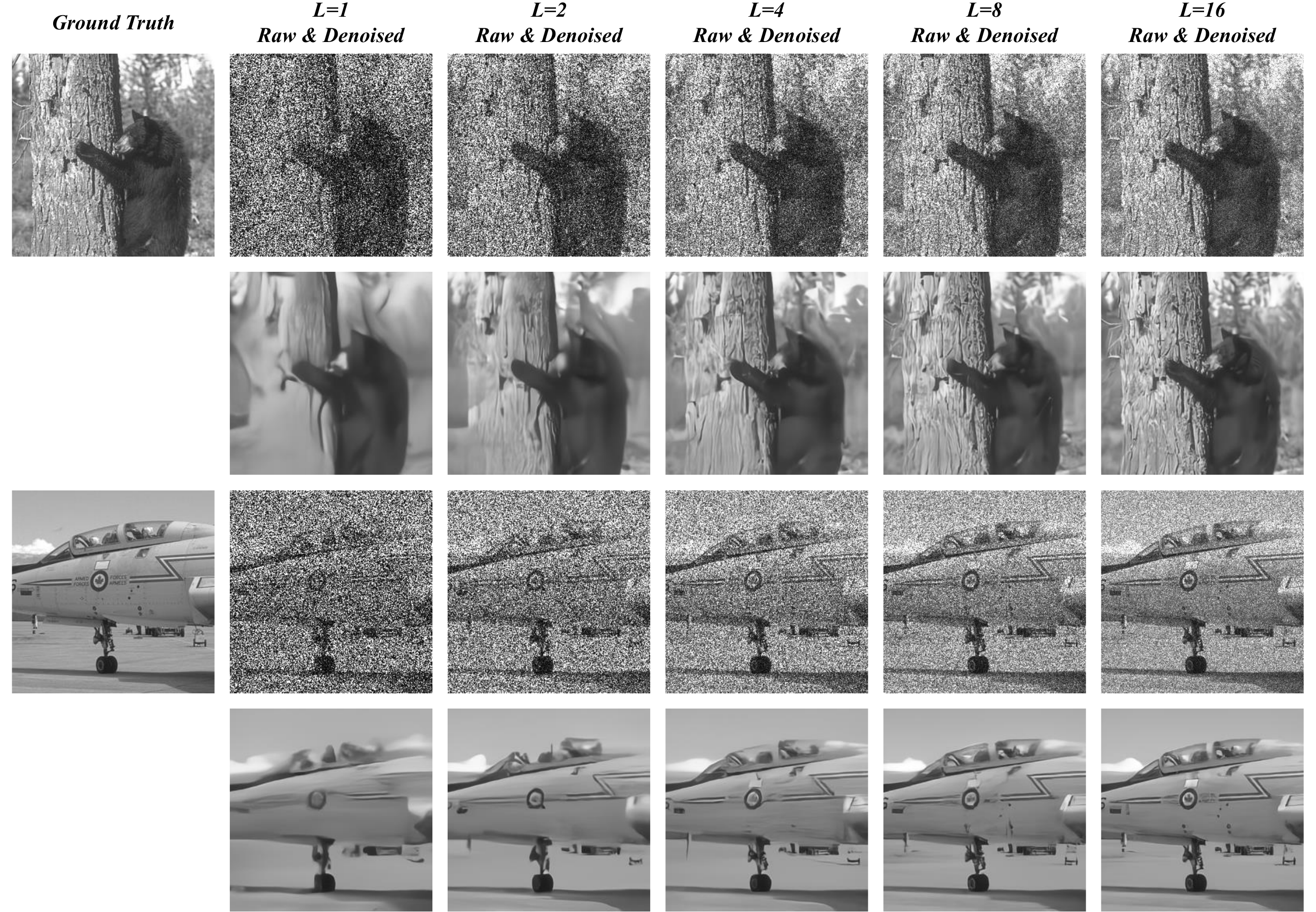}
		\caption{Qualitative synthetic despeckling on two natural scenes.
			Left: clean reference. Remaining columns: noisy Gamma-corrupted
			input and $\gamma$-Bridge output at $\Lin \in \{1,2,4,8,16\}$ from
			a single $\Lobs{=}1$-trained model (5-step deterministic, smart-start).}
		
		\label{fig:cmp-synth}
	\end{figure}
	
	\begin{table}[t]
		\scriptsize
		\centering
		\caption{Set12: PSNR / SSIM$\times100$ vs.\ $\Lobs$. Bold: best per
			column; underline: second.}
		\label{tab:cmp-set12}
		\begin{tabularx}{\linewidth}{p{1.6cm}*{8}{>{\centering\arraybackslash}X}}
			\toprule[1pt]
			\multirow{2}{*}{} & \multicolumn{2}{c}{\textbf{1-Look}} & \multicolumn{2}{c}{\textbf{2-Look}} & \multicolumn{2}{c}{\textbf{4-Look}} & \multicolumn{2}{c}{\textbf{8-Look}} \\
			& \multicolumn{1}{>{\hsize=1.5\hsize}X}{\textbf{PSNR}} & \multicolumn{1}{>{\hsize=1.5\hsize}X}{\textbf{SSIM}} & \multicolumn{1}{>{\hsize=1.5\hsize}X}{\textbf{PSNR}} & \multicolumn{1}{>{\hsize=1.5\hsize}X}{\textbf{SSIM}} & \multicolumn{1}{>{\hsize=1.5\hsize}X}{\textbf{PSNR}} & \multicolumn{1}{>{\hsize=1.5\hsize}X}{\textbf{SSIM}} & \multicolumn{1}{>{\hsize=1.5\hsize}X}{\textbf{PSNR}} & \multicolumn{1}{>{\hsize=1.5\hsize}X}{\textbf{SSIM}} \\
			\midrule
			ANLM~\cite{4359947}         & 14.11 & 27.86 & 17.97 & 40.04 & 23.04 & 60.08 & 22.94 & 66.98 \\
			DIP~\cite{Ulyanov_2018_CVPR}           & 17.08 & 48.69 & 18.18 & 56.81 & 19.04 & 61.86 & 20.83 & 68.97 \\
			DnCNN~\cite{7839189}       & 18.36 & 39.91 & 20.72 & 46.02 & 22.43 & 51.99 & 23.44 & 58.01 \\
			SARCAM~\cite{ko2021sar}     & 16.25 & 27.15 & 22.56 & 57.74 & 24.29 & 68.34 & 23.41 & 70.61 \\
			SAR2SAR~\cite{sar2sar}   & 17.54 & 45.06 & 21.78 & 55.39 & 22.08 & 59.74 & 22.22 & 61.72 \\
			SARtrans~\cite{9884596} & \underline{18.84} & 51.58 & 19.11 & 52.15 & 18.71 & 52.30 & 18.55 & 52.51 \\
			AGSDNet~\cite{9755131}   & 13.46 & 18.71 & 21.29 & 48.10 & \textbf{26.14} & \textbf{75.33} & \textbf{27.76} & \textbf{81.23} \\
			SIFSDNet~\cite{9883415} & 18.66 & 38.67 & \textbf{23.11} & 61.95 & \underline{25.14} & 72.01 & 25.82 & 76.87 \\
			MONet~\cite{9324183}       & \textbf{21.05} & \textbf{54.49} & \underline{22.77} & \underline{65.74} & 23.50 & 66.87 & 24.02 & 67.26 \\
			RDDPM~\cite{10641283}       & 16.98 & 32.22 & 18.50 & 38.51 & 19.93 & 44.26 & 21.03 & 49.87 \\
			CL-SAR~\cite{FANG2024376}      & 15.52 & 29.07 & 16.61 & 55.30 & 17.97 & 71.16 & 19.11 & 77.11 \\
			S$^3$DIP~\cite{albisani2025self}    & 16.38 & 50.98 & 17.30 & 56.06 & 17.00 & 62.57 & 16.26 & 64.35 \\
			$\gamma$-Bridge               & 18.77 & \underline{53.86} & 20.88 & \textbf{67.56} & 23.56 & \underline{74.56} & \underline{26.10} & \underline{79.80} \\
			\bottomrule[1pt]
		\end{tabularx}
	\end{table}

	\smallskip
	
	\noindent\textbf{Two-step consistency.}
	$\mathcal{L}_{\mathrm{rec}}$ alone is a one-step MAE objective: gradients flow only through a single application of $G_\theta$, so the multi-step reverse trajectory receives no explicit training signal. To induce trajectory-consistent predictions, we perform a second forward pass at an intermediate step. Given the first prediction $\hat{x}_0^{(1)} = \hat{x}_0(x_t; t, \xobs)$, we form
	\begin{equation}
		\begin{aligned}
			x_{t^{\prime}} \;&=\; \tilde{\alpha}\, x_t \;+\; (1 - \tilde{\alpha})\,\mathrm{sg}\!\bigl(\hat{x}_0^{(1)}\bigr), \\
			\tilde{\alpha} \;&=\; L(t)/L(t^{\prime}),
		\end{aligned}
		\label{eq:cons-x-mid}
	\end{equation}
	with $t^{\prime} \sim \mathrm{Unif}\{0,\ldots,t{-}1\}$ (a step strictly closer to the clean end, so $L(t^{\prime}) > L(t)$), $\mathrm{sg}$ the stop-gradient operator, and $\tilde{\alpha} = L(t)/L(t^{\prime}) \in (0,1)$ obtained by applying Eq.~\ref{eq:posterior-otode} with $t\!\to\!t^{\prime}$ (i.e.\ writing $\alpha_{t\to t^{\prime}} := L(t)/L(t^{\prime})$ in analogy with $\alpha_t = L(t)/L(t{-}1)$). The consistency loss is
	\begin{equation}
		\mathcal{L}_{\mathrm{cons}}
		\;=\; \EE\,\bigl\|\hat{x}_0(x_{t^{\prime}}; t^{\prime}, \xobs) - x_0\bigr\|_1.
		\label{eq:loss-cons}
	\end{equation}
	The stop-gradient on $\hat{x}_0^{(1)}$ prevents the first prediction from being trivially pulled toward the second, preserving the interpretation of the loss as a trajectory-consistency check rather than as a longer computation graph on the same target.
	
	\begin{remark}[Gradient contribution]\label{rem:cons-gradient}
		Appendix~\ref{app:cons-gradient} shows that $\mathcal{L}_{\mathrm{cons}}$ provides a training signal that $\mathcal{L}_{\mathrm{rec}}$ cannot: at the pointwise $\ell_1$ optimum, $\nabla_\theta \mathcal{L}_{\mathrm{rec}}{=}0$, whereas $\mathcal{L}_{\mathrm{cons}}$ retains a non-trivial gradient along the reverse trajectory.
	\end{remark}
	
	\noindent \textbf{Full objective:}
	\begin{equation}
		\begin{aligned}
			\mathcal{L} \;=\; \lambda_{\mathrm{rec}}\mathcal{L}_{\mathrm{rec}}
			\;&+\; \lambda_{\mathrm{ratio}}\mathcal{L}_{\mathrm{ratio}} \\
			\;&+\; \lambda_{\mathrm{cons}}\mathcal{L}_{\mathrm{cons}}.
		\end{aligned}
		\label{eq:loss-total}
	\end{equation}
	We use $\lambda_{\mathrm{rec}}=10$, $\lambda_{\mathrm{ratio}}=1$, $\lambda_{\mathrm{cons}}=5$ throughout the experiments unless otherwise noted.
	
	\subsection{Runtime look-number control}
	\label{sec:method-runtime}
	
	At inference time the trained network exposes two orthogonal control axes.
	
	\begin{table}[t]
		\scriptsize
		\centering
		\caption{McMaster: PSNR / SSIM$\times100$ vs.\ $\Lobs$.}
		\label{tab:cmp-mcmaster}
		\begin{tabularx}{\linewidth}{p{1.6cm}*{8}{>{\centering\arraybackslash}X}}
			\toprule[1pt]
			\multirow{2}{*}{} & \multicolumn{2}{c}{\textbf{1-Look}} & \multicolumn{2}{c}{\textbf{2-Look}} & \multicolumn{2}{c}{\textbf{4-Look}} & \multicolumn{2}{c}{\textbf{8-Look}} \\
			& \multicolumn{1}{>{\hsize=1.5\hsize}X}{\textbf{PSNR}} & \multicolumn{1}{>{\hsize=1.5\hsize}X}{\textbf{SSIM}} & \multicolumn{1}{>{\hsize=1.5\hsize}X}{\textbf{PSNR}} & \multicolumn{1}{>{\hsize=1.5\hsize}X}{\textbf{SSIM}} & \multicolumn{1}{>{\hsize=1.5\hsize}X}{\textbf{PSNR}} & \multicolumn{1}{>{\hsize=1.5\hsize}X}{\textbf{SSIM}} & \multicolumn{1}{>{\hsize=1.5\hsize}X}{\textbf{PSNR}} & \multicolumn{1}{>{\hsize=1.5\hsize}X}{\textbf{SSIM}} \\
			\midrule
			ANLM~\cite{4359947}         & 11.72 & 22.21 & 13.91 & 30.65 & 16.37 & 39.98 & 19.37 & 50.49 \\
			DIP~\cite{Ulyanov_2018_CVPR}           & 16.09 & 52.75 & 18.70 & 61.90 & 20.29 & 67.68 & 22.46 & 72.67 \\
			DnCNN~\cite{7839189}       & 11.83 & 22.86 & 13.86 & 31.98 & 16.56 & 43.12 & 20.08 & 54.77 \\
			SARCAM~\cite{ko2021sar}     & 14.33 & 26.44 & 17.78 & 39.64 & 24.08 & 64.60 & 27.43 & 79.20 \\
			SAR2SAR~\cite{sar2sar}   & 13.16 & 35.75 & 20.59 & 54.99 & 22.46 & 64.93 & 23.09 & 68.59 \\
			SARtrans~\cite{9884596} & 18.40 & 52.42 & 19.56 & 53.22 & 19.85 & 53.63 & 19.86 & 53.83 \\
			AGSDNet~\cite{9755131}   & 12.18 & 18.11 & 15.14 & 27.49 & 22.66 & 55.87 & \underline{28.49} & 81.67 \\
			SIFSDNet~\cite{9883415} & 16.01 & 30.62 & 20.04 & 48.33 & \underline{24.90} & 69.54 & 28.33 & \underline{82.02} \\
			MONet~\cite{9324183}       & 15.15 & 26.57 & 21.48 & 60.08 & 24.44 & \underline{77.36} & 26.06 & 78.75 \\
			RDDPM~\cite{10641283}       & 13.31 & 25.71 & 14.88 & 30.01 & 16.42 & 33.76 & 17.71 & 36.77 \\
			CL-SAR~\cite{FANG2024376}      & 15.77 & 36.12 & 17.11 & 60.60 & 18.95 & 76.29 & 20.37 & 81.75 \\
			S$^3$DIP~\cite{albisani2025self}    & 14.96 & \underline{55.98} & 15.14 & \underline{64.40} & 15.15 & 64.85 & 15.30 & 62.62 \\
			$\gamma$-Bridge               & \textbf{20.34} & \textbf{63.07} & \textbf{23.71} & \textbf{76.78} & \textbf{26.35} & \textbf{82.55} & \textbf{28.89} & \textbf{86.93} \\
			\bottomrule[1pt]
		\end{tabularx}
	\end{table}
	
	\begin{table}[t]
		\scriptsize
		\centering
		\caption{Kodak24: PSNR / SSIM$\times100$ vs.\ $\Lobs$.}
		\label{tab:cmp-kodak24}
		\begin{tabularx}{\linewidth}{p{1.6cm}*{8}{>{\centering\arraybackslash}X}}
			\toprule[1pt]
			\multirow{2}{*}{} & \multicolumn{2}{c}{\textbf{1-Look}} & \multicolumn{2}{c}{\textbf{2-Look}} & \multicolumn{2}{c}{\textbf{4-Look}} & \multicolumn{2}{c}{\textbf{8-Look}} \\
			& \multicolumn{1}{>{\hsize=1.5\hsize}X}{\textbf{PSNR}} & \multicolumn{1}{>{\hsize=1.5\hsize}X}{\textbf{SSIM}} & \multicolumn{1}{>{\hsize=1.5\hsize}X}{\textbf{PSNR}} & \multicolumn{1}{>{\hsize=1.5\hsize}X}{\textbf{SSIM}} & \multicolumn{1}{>{\hsize=1.5\hsize}X}{\textbf{PSNR}} & \multicolumn{1}{>{\hsize=1.5\hsize}X}{\textbf{SSIM}} & \multicolumn{1}{>{\hsize=1.5\hsize}X}{\textbf{PSNR}} & \multicolumn{1}{>{\hsize=1.5\hsize}X}{\textbf{SSIM}} \\
			\midrule
			ANLM~\cite{4359947}         & 11.29 & 10.03 & 13.49 & 16.32 & 16.13 & 26.33 & 19.27 & 40.46 \\
			DIP~\cite{Ulyanov_2018_CVPR}           & \underline{19.01} & 50.47 & 20.51 & 55.32 & 22.02 & 59.86 & 23.37 & 64.39 \\
			DnCNN~\cite{7839189}       & 11.41 & 10.30 & 13.35 & 16.55 & 16.04 & 27.60 & 19.62 & 43.76 \\
			SARCAM~\cite{ko2021sar}     & 13.85 & 15.47 & 16.97 & 26.39 & 22.64 & 51.64 & 25.38 & 69.12 \\
			SAR2SAR~\cite{sar2sar}   & 16.95 & 49.35 & 20.10 & 55.39 & 20.76 & 57.76 & 20.95 & 59.13 \\
			SARtrans~\cite{9884596} & 17.95 & 48.42 & 19.14 & 48.91 & 19.63 & 49.13 & 19.77 & 49.19 \\
			AGSDNet~\cite{9755131}   & 11.80 & 10.60 & 14.66 & 17.55 & 21.29 & 43.05 & \textbf{26.30} & 72.83 \\
			SIFSDNet~\cite{9883415} & 15.55 & 19.81 & 19.19 & 36.09 & \underline{23.48} & 58.20 & \textbf{26.30} & 73.14 \\
			MONet~\cite{9324183}       & 14.71 & 17.44 & \underline{20.67} & 47.89 & 23.45 & 68.00 & 24.59 & 67.78 \\
			RDDPM~\cite{10641283}       & 13.03 & 21.41 & 14.60 & 26.34 & 16.08 & 30.85 & 17.40 & 34.36 \\
			CL-SAR~\cite{FANG2024376}      & 15.91 & 23.78 & 18.37 & 50.32 & 20.36 & \underline{68.74} & 22.36 & \underline{75.09} \\
			S$^3$DIP~\cite{albisani2025self}    & 17.47 & \underline{51.96} & 16.50 & \underline{58.01} & 15.91 & 58.02 & 16.02 & 56.57 \\
			$\gamma$-Bridge               & \textbf{19.41} & \textbf{52.25} & \textbf{22.00} & \textbf{65.58} & \textbf{24.00} & \textbf{72.33} & \underline{25.85} & \textbf{77.55} \\
			\bottomrule[1pt]
		\end{tabularx}
	\end{table}
	
	\begin{figure*}[t]
		\centering
		\includegraphics[width=0.75\linewidth]{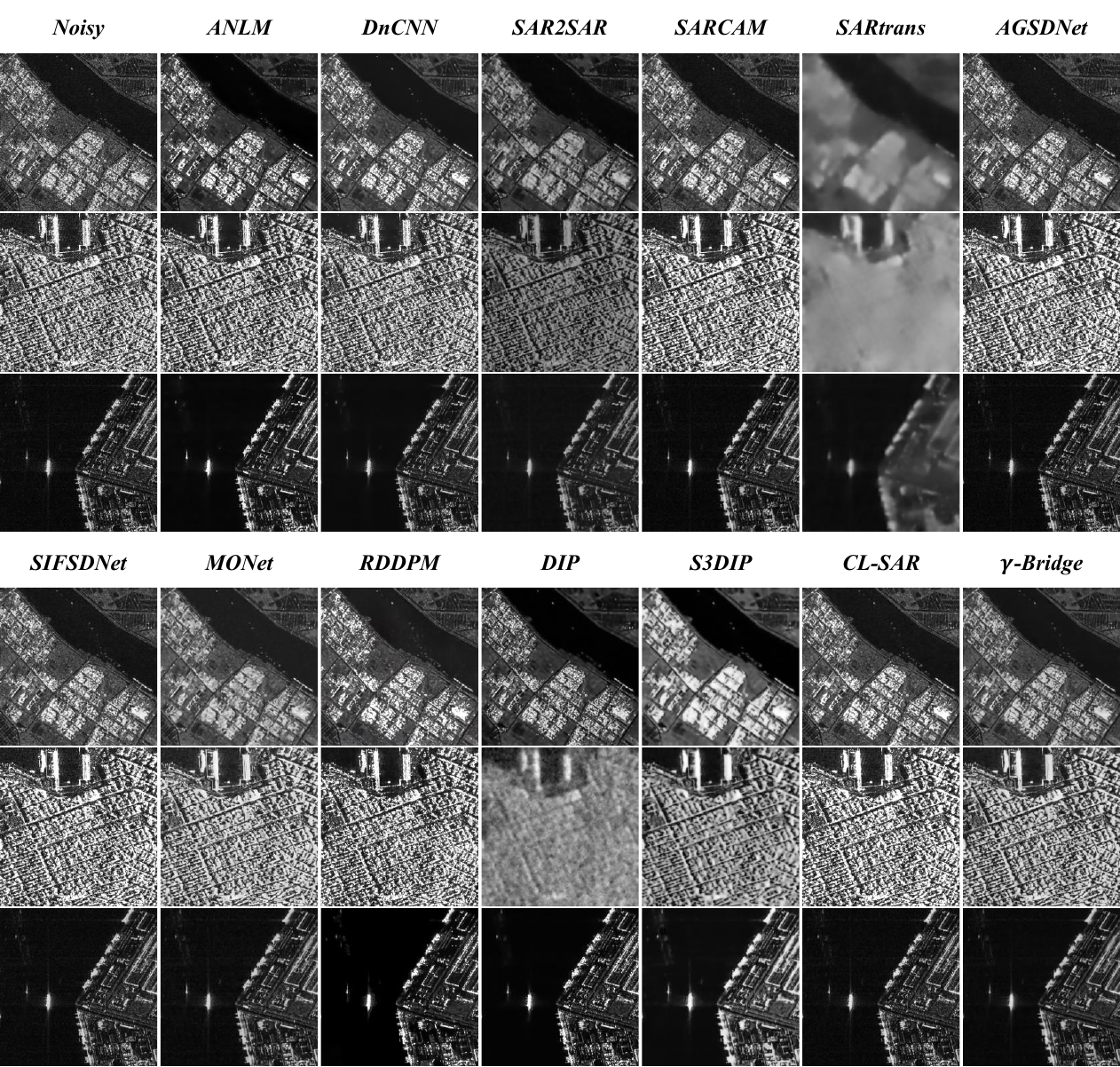}
		\caption{Qualitative comparison on spaceborne SAR sensors. Thirteen methods are split into two rows of seven columns each; $\gamma$-Bridge is the rightmost panel in the bottom row. Sensor rows (repeated in both halves) from top to bottom: Sentinel-1, TerraSAR-X, Gaofen-3.}
		\label{fig:cmp-real1}
	\end{figure*}
	
	\begin{figure*}[t]
		\centering
		\includegraphics[width=0.8\linewidth]{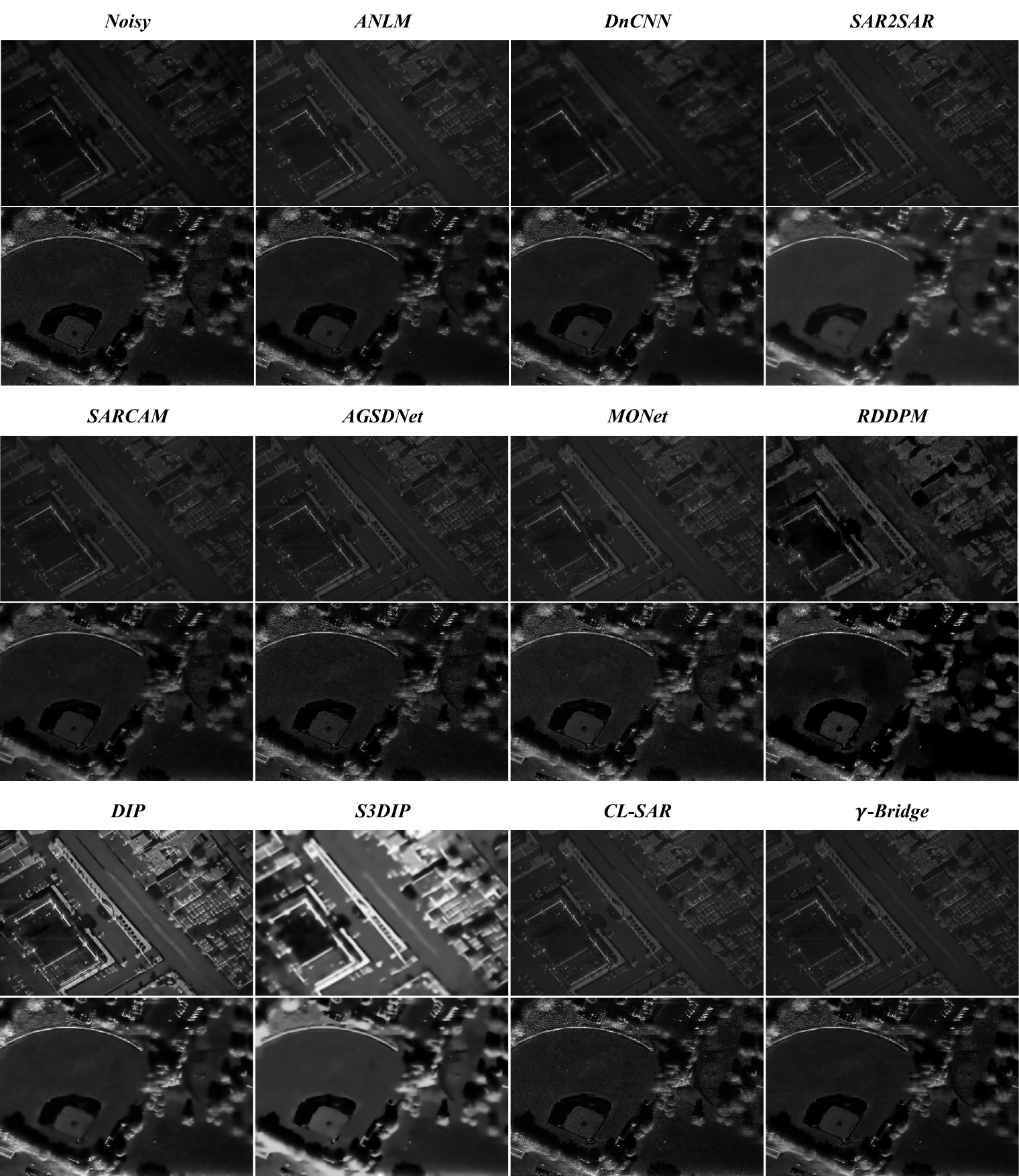}
		\caption{Qualitative comparison on airborne SAR sensors. Eleven methods are split into three rows; $\gamma$-Bridge is the rightmost panel in the bottom row. Sensor rows (repeated across method groups): miniSAR (top) and FARAD X-band (bottom).}
		\label{fig:cmp-real2}
	\end{figure*}
	
	\begin{figure*}[t]
		\centering
		\includegraphics[width=0.8\textwidth]{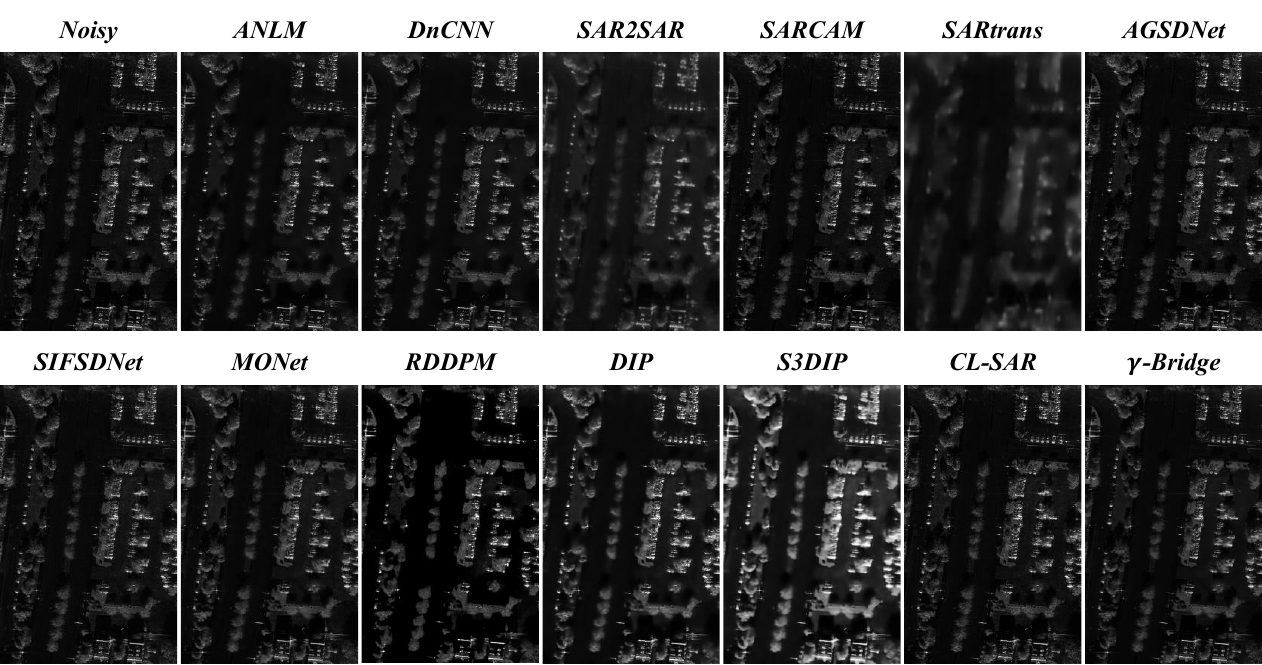}
		\caption{Qualitative comparison on a FARAD X-band scene (parking lot with vehicles). Thirteen methods split across two rows; $\gamma$-Bridge is the rightmost panel in the bottom row.}
		\label{fig:cmp-real3}
	\end{figure*}
	
	\begin{table*}[!t]
		\scriptsize
		\centering
		\caption{Spaceborne SAR: no-reference metrics (ENL, EPI, SQI, Mean). \label{tab:cmp-spaceborne}}
		\begin{tabularx}{0.8\linewidth}{>{\raggedright\arraybackslash}p{1.6cm}CCCCCCCCCCCC}
			\toprule
			\multirow{2}{*}{} &
			\multicolumn{4}{c}{\textbf{Sentinel-1} (Mean$=$34.39)} &
			\multicolumn{4}{c}{\textbf{TerraSAR-X} (Mean$=$30.04)} &
			\multicolumn{4}{c}{\textbf{Gaofen-3} (Mean$=$20.10)} \\
			& ENL & EPI & SQI & Mean
			& ENL & EPI & SQI & Mean
			& ENL & EPI & SQI & Mean \\
			\midrule
			ANLM~\cite{4359947}
			& 1.623 & \textbf{0.719} & \textbf{0.787} & 29.99
			& 1.413 & 0.553 & 0.676 & 29.13
			& 0.930 & 0.609 & 0.646 & 20.87 \\
			DIP~\cite{Ulyanov_2018_CVPR}
			& 2.766 & 0.117 & 0.601 & 31.89
			& \textbf{4.197} & 0.063 & 0.604 & 33.42
			& 0.670 & 0.081 & 0.445 & 18.37 \\
			DnCNN~\cite{7839189}
			& 2.781 & 0.618 & 0.589 & 35.67
			& 1.546 & 0.550 & 0.732 & 30.09
			& 1.177 & 0.689 & 0.766 & \underline{20.23} \\
			SARCAM~\cite{ko2021sar}
			& 2.816 & 0.621 & 0.595 & 34.75
			& 1.696 & 0.548 & 0.787 & 30.85
			& 1.231 & 0.603 & 0.784 & \textbf{20.01} \\
			SAR2SAR~\cite{sar2sar}
			& 2.868 & 0.651 & 0.645 & 27.46
			& 2.236 & 0.276 & \underline{0.867} & 22.85
			& 1.103 & 0.474 & 0.747 & 16.67 \\
			AGSDNet~\cite{9755131}
			& 2.738 & 0.626 & 0.600 & \underline{34.84}
			& 1.463 & \underline{0.697} & 0.697 & \underline{30.60}
			& 1.137 & \underline{0.847} & 0.743 & 19.89 \\
			SIFSDNet~\cite{9883415}
			& 3.019 & 0.599 & 0.567 & 35.70
			& 1.716 & 0.614 & 0.796 & 31.48
			& 1.290 & 0.721 & \underline{0.812} & 20.86 \\
			MONet~\cite{9324183}
			& \textbf{4.548} & 0.497 & 0.455 & 37.71
			& 2.268 & 0.457 & \textbf{0.885} & 32.73
			& \textbf{1.577} & 0.638 & \textbf{0.823} & 21.96 \\
			RDDPM~\cite{10641283}
			& 2.795 & 0.614 & 0.583 & 37.17
			& 1.194 & 0.495 & 0.572 & 28.60
			& 0.788 & 0.661 & 0.657 & 18.82 \\
			CL-SAR~\cite{FANG2024376}
			& 2.802 & 0.226 & 0.619 & 37.90
			& 1.494 & 0.266 & 0.544 & 32.53
			& 1.057 & 0.200 & 0.498 & 20.37 \\
			S$^3$DIP~\cite{albisani2025self}
			& 3.767 & 0.081 & 0.643 & 44.35
			& 2.195 & 0.088 & 0.593 & 35.69
			& 0.909 & 0.065 & 0.481 & 22.24 \\
			$\gamma$-Bridge
			& \underline{3.812} & \underline{0.652} & \underline{0.647} & \textbf{34.20}
			& \underline{2.352} & \textbf{0.734} & 0.580 & \textbf{29.84}
			& \underline{1.404} & \textbf{0.913} & 0.629 & 19.91 \\
			\bottomrule
		\end{tabularx}
	\end{table*}
	
	\begin{table*}[!t]
		\scriptsize
		\centering
		\caption{Airborne SAR: no-reference metrics (ENL, EPI, SQI, Mean).
			\label{tab:cmp-airborne}}
		\begin{tabularx}{0.8\linewidth}{>{\raggedright\arraybackslash}p{1.6cm}CCCCCCCCCCCC}
			\toprule
			\multirow{2}{*}{} &
			\multicolumn{4}{c}{\textbf{miniSAR} (Mean$=$13.32)} &
			\multicolumn{4}{c}{\textbf{FARAD-X} (Mean$=$10.05)} &
			\multicolumn{4}{c}{\textbf{FARAD-Ka} (Mean$=$9.63)} \\
			& ENL & EPI & SQI & Mean
			& ENL & EPI & SQI & Mean
			& ENL & EPI & SQI & Mean \\
			\midrule
			ANLM~\cite{4359947}
			& 4.438 & 0.172 & 0.519 & \phantom{0}9.82
			& 0.765 & 0.343 & 0.658 & \phantom{0}9.58
			& 1.151 & 0.305 & 0.803 & \phantom{0}9.21 \\
			DIP~\cite{Ulyanov_2018_CVPR}
			& 3.817 & 0.028 & 0.655 & 29.24
			& 0.988 & 0.015 & 0.477 & 13.65
			& 1.584 & 0.028 & 0.503 & 14.41 \\
			DnCNN~\cite{7839189}
			& 8.538 & 0.141 & 0.726 & 13.52
			& 0.847 & 0.330 & 0.703 & \underline{10.24}
			& 1.225 & 0.296 & 0.837 & \phantom{0}9.82 \\
			SARCAM~\cite{ko2021sar}
			& \textbf{9.204} & 0.023 & \textbf{0.755} & 12.00
			& 0.952 & 0.080 & 0.752 & \phantom{0}9.46
			& 1.099 & 0.359 & 0.773 & \phantom{0}9.28 \\
			SAR2SAR~\cite{sar2sar}
			& 5.872 & 0.023 & 0.592 & 12.97
			& \textbf{2.343} & 0.052 & \textbf{0.888} & 14.93
			& \textbf{2.061} & 0.081 & \textbf{0.928} & 10.99 \\
			AGSDNet~\cite{9755131}
			& 6.013 & \textbf{0.638} & 0.610 & \textbf{13.24}
			& 0.668 & \underline{0.887} & 0.606 & \phantom{0}9.86
			& 0.967 & \underline{0.680} & 0.724 & \phantom{0}\underline{9.46} \\
			SIFSDNet~\cite{9883415}
			& \underline{8.680} & 0.074 & 0.733 & 14.04
			& 1.056 & 0.104 & \underline{0.807} & 10.54
			& 1.199 & 0.520 & 0.827 & 10.09 \\
			MONet~\cite{9324183}
			& 8.449 & 0.132 & 0.724 & 14.30
			& 0.929 & 0.261 & 0.746 & 11.05
			& 1.382 & 0.312 & \underline{0.887} & 10.47 \\
			RDDPM~\cite{10641283}
			& 2.385 & \underline{0.545} & 0.382 & 11.65
			& 0.379 & 0.389 & 0.406 & \phantom{0}7.94
			& 0.563 & 0.360 & 0.449 & \phantom{0}7.41 \\
			CL-SAR~\cite{FANG2024376}
			& 7.478 & 0.016 & 0.730 & 15.64
			& 0.685 & 0.111 & 0.445 & 10.65
			& 0.939 & 0.097 & 0.462 & \phantom{0}9.94 \\
			S$^3$DIP~\cite{albisani2025self}
			& 3.744 & 0.006 & 0.656 & 42.40
			& 1.232 & 0.008 & 0.512 & 20.61
			& \underline{1.524} & 0.026 & 0.530 & 20.46 \\
			$\gamma$-Bridge
			& 6.743 & 0.491 & \underline{0.736} & \underline{13.13}
			& \underline{1.480} & \textbf{0.910} & 0.559 & \textbf{\phantom{0}9.87}
			& 1.251 & \textbf{0.720} & 0.588 & \textbf{\phantom{0}9.53} \\
			\bottomrule
		\end{tabularx}
	\end{table*}
	
	\noindent \textbf{Target-$L$ output control (stopping).}
	To produce a $\Lout$-look output for any $\Lout \in [\Lobs, \Lmax]$, select $t^\star = \arg\min_t |L(t) - \Lout|$ and terminate the reverse chain at $t^\star$ rather than at step $0$. Every intermediate $x_{t}$ is a physically valid $L(t)$-look image (Proposition~\ref{prop:forward-marg}), and the deterministic reverse tracks the mean of $x_0 \cdot \Gam(\Lout, \Lout)$.
	
	\smallskip
	
	\noindent \textbf{Smart-start input control.}
	Symmetrically, given an input $y$ at look number $\Lin \in [\Lobs, \Lmax]$ (the case $\Lin = \Lobs$ reduces to the standard reverse from $t^\star = T{-}1$), select $t^\star = \arg\min_t |L(t) - \Lin|$ and initialise the reverse at $x_{t^\star} \leftarrow y$, running the chain from $t^\star$ down to the desired output step. Because $y$ is by construction an in-distribution sample of $q(x_{t^\star} \mid x_0)$, the reverse is correctly anchored without any retraining.
	
	Combining the two, a single training at $\Lobs$ supports zero-shot  denoising for any $(\Lin, \Lout)$ pair with $\Lobs \le \Lin \le \Lmax$ and $\Lin \le \Lout \le \Lmax$.
	
	\smallskip
	
	\noindent \textbf{Estimating $\Lin$ for real deployment.}
	Smart-start requires an estimate of the input's effective look number. On synthetic data $\Lin$ is known by construction; on real SAR the sensor and processing chain (SLC, multi-look GRD, temporal averages) fix a nominal $\Lin$ that can vary across scenes and is rarely reported alongside individual tiles. We use a homogeneous-patch estimator: slide $32{\times}32$ windows at stride $16$ over the image, compute the local ENL $(\bar{x}/\hat{\sigma}_x)^2$ per patch, and take the $90^{\text{th}}$ quantile as $\hat L$. The top-decile patches are the flattest regions in the scene, where measured variance is dominated by speckle rather than by underlying reflectivity structure, so the estimator tracks the effective look number rather than the scene-plus-noise variance picked up by a whole-image estimator (which is biased downward by scene texture). At inference we set $\Lin = \hat L$ and smart-start at $t^\star = \arg\min_t |L(t) - \hat L|$. The estimator is a one-shot O(HW) computation that adds negligible overhead relative to the reverse chain, and it is only needed at inference; the training loop is unchanged.

	
	\section{Experiments}
	\label{sec:exp}
	
	\subsection{Datasets and Metrics}
	
	\noindent \textbf{Training data.} Training uses only BSDS500 train+val and DIV2K train HR natural images: $1{,}100$ $256{\times}256$-pixel crops. Eq.~\ref{eq:speckle} generates $\Lobs = 1$ observations; no real SAR imagery is used, matching prior synthetic regimes ~\cite{guha2023sddpm,sarcnn,heo2026self}. 
	
	\smallskip
	
	\noindent \textbf{Synthetic evaluation.} Evaluation uses 64 in-domain held-out BSDS500 test crops at $\Lobs=1$ and cross-domain Kodak24, Set12, and McMaster images synthesized by Eq.~\ref{eq:speckle} at $\Lobs\in\{1,2,4,8,16\}$, testing generalization beyond the training distribution.
	
	\smallskip
	
	\noindent \textbf{Real-SAR evaluation.} Zero-shot tests cover 33 tiles from six spaceborne/airborne C-/X-/Ka-band sensors: Sentinel-1 IW, TerraSAR-X, Gaofen-3, miniSAR, FARAD Ka-band, FARAD X-band. Following~\cite{10109106,guha2023sddpm,heo2026self}, released PNG/JPG products serve as intensities. Section~\ref{sec:method-runtime} provides tile-wise $\hat L$ for smart-start, without domain adaptation.
	
	\smallskip
	
	\noindent\textbf{Metrics.}
	Synthetic evaluation uses PSNR/SSIM against $x_0$. Distributional evaluation assesses the residual ratio $r = \xobs/\hat{x}_0$ via its mean, variance, and a Kolmogorov--Smirnov test against $\Gam(\Lobs,\Lobs)$. Real-SAR evaluation is no-reference and uses four homogeneous-patch metrics computed over $32{\times}32$ windows at stride $16$ ($90^{\mathrm{th}}$ percentile across patches): $\mathrm{ENL} = (\mu/\sigma)^2$, $\mathrm{EPI}$, $\mathrm{SQI}$, and $\mathrm{Mean}$.
	
	\smallskip
	
	\noindent \textbf{Model.}
	A fully convolutional U-Net with 64 base channels, multipliers $(1,2,4)$, and no positional embeddings is trained on one RTX~3080 GPU. An exponential $T{=}100$ schedule spans $\Lobs{=}1$ to $\Lmax{=}10^4$, with $(\lambda_{\mathrm{rec}},\lambda_{\mathrm{ratio}},\lambda_{\mathrm{cons}})=(10,1,5)$. Deterministic reverse is the sampling default.

	\begin{table}[t]
		\centering
		\scriptsize
		\caption{Multi-step reverse at $\Lobs{=}1$ on 64 BSDS500 crops
			(deterministic, ckpt\_15k EMA).}
		\label{tab:multistep-l1}
		\begin{tabular}{@{}lcccc@{}}
			\toprule
			NFE & PSNR (dB) & SSIM & ratio mean & KS $p$ \\
			\midrule
			1  & 21.40 & 0.446 & 0.994 & 1.3e-5 \\
			\textbf{5}  & \textbf{22.56} & \textbf{0.556} & \textbf{0.998} & \textbf{0.42} \\
			25 & 22.05 & 0.551 & 0.947 & 1.3e-14 \\
			\bottomrule
		\end{tabular}
	\end{table}

	\subsection{Comparison with prior work}
	\label{sec:exp-baselines}
	
	Twelve methods are compared on Kodak24, Set12, McMaster, and Sentinel-1, Gaofen-3, TerraSAR-X, miniSAR/FARAD-X/FARAD-Ka real-SAR tiles. Identical-release baselines~\cite{hu2024sar} comprise ANLM~\cite{4359947}, DIP~\cite{Ulyanov_2018_CVPR}, DnCNN~\cite{7839189}, SARCAM~\cite{ko2021sar}, AGSDNet~\cite{9755131}, SIFSDNet~\cite{9883415}, MONet~\cite{9324183}, SARtrans~\cite{9884596}, SAR2SAR~\cite{sar2sar}, CL-SAR~\cite{FANG2024376}, S$^3$DIP~\cite{albisani2025self}, and RDDPM~\cite{10641283}. $\gamma$-Bridge uses one BSDS500+DIV2K $\Lobs{=}1$ checkpoint; smart-start uses true synthetic $\Lobs$ or Section~\ref{sec:method-runtime}'s real-SAR homogeneous-patch $\hat L$, followed by post-hoc normalization. No fine-tuning, sensor-specific tuning, or clean references are used.
	
	\smallskip
	
	\noindent \textbf{Synthetic.} Tables~\ref{tab:cmp-set12}--\ref{tab:cmp-kodak24} report PSNR (dB) and SSIM ($\times100$) for $\Lobs\in\{1,2,4,8\}$. Figure~\ref{fig:cmp-synth} shows finer recovery from $\Lin$ $1$ to $16$ by one $\Lobs{=}1$-trained model.
	
	Across three benchmarks, \(\gamma\)-Bridge leads McMaster/Kodak24 for \(\Lobs \le 4\) and is competitive on Set12. At \(\Lobs=1\), $\gamma$-Bridge matches or leads on the former two but trails nonlocal-sparse methods on piecewise-smooth Set12, which favors sparse priors. Its McMaster \(\Lobs=8\) margins over the runner-up---\(0.40\) dB and \(4.91\) SSIM---demonstrate generalization to rich photographic textures.
	
	\smallskip
	
	\noindent \textbf{Real SAR.} Following~\cite{hu2024sar}, Tables~\ref{tab:cmp-spaceborne} and~\ref{tab:cmp-airborne} give ENL, EPI, SQI, and Mean; Figures~\ref{fig:cmp-real1}--\ref{fig:cmp-real3} compare sensors. EPI retains its reference name; all denoised-output metrics use the exact reference formulas.
	
	\smallskip
	
	\noindent \textbf{Trade-off summary.}
	For $\Lobs \in \{1,2,4,8\}$ on three synthetic benchmarks, $\gamma$-Bridge leads McMaster/Kodak24 PSNR at $\Lobs \le 4$ and SSIM at all $\Lobs$ levels (e.g.\ $\Lobs{=}8$: McMaster $28.89$\,dB / $86.9$; Kodak24 $25.85$\,dB / $77.6$). $\gamma$-Bridge is competitive on Set12 at $\Lobs \ge 2$ but trails MONet at $\Lobs{=}1$ ($18.77$ vs.\ $21.05$\,dB), where piecewise-smooth content favors nonlocal sparse priors over photograph-trained diffusion.
	
	Real-SAR performance is sensor-dependent. On the three C/X-band spaceborne sensors closest to the synthetic training-noise model, $\gamma$-Bridge ranks second in ENL across the board: Sentinel-1 ($3.81$ vs.\ MONet $4.55$), TerraSAR-X ($2.35$ vs.\ DIP $4.20$), and Gaofen-3 ($1.40$ vs.\ MONet $1.58$), while preserving noisy-input means (Mean\,=\,$34.20$, $29.84$, $19.91$, all within $1\%$ of the noisy references). On the three airborne sensors $\gamma$-Bridge is mid-pack: it outperforms SAR2SAR on miniSAR ($6.74$ vs.\ $5.87$), ranks second only on FARAD-X ($1.48$ vs.\ SAR2SAR $2.34$), and trails SAR2SAR on FARAD-Ka. The remaining gap reflects out-of-model high-frequency structure and dark-tail intensities characteristic of high-resolution airborne SAR; repetitive fine texture also over-shoots the smart-start estimator $\hat L$. Small-set airborne fine-tuning or a randomized-$\Lobs$ variant (Appendix~\ref{app:random-Lobs}) may close it but remains future work.
	
	The consistently lower SQI than edge-preserving sparse baselines reflects $\gamma$-Bridge's smart-start reverse suppressing noisy input gradients, which reduces overall output-gradient magnitude---this trades off against the gradient-based EPI metric but does not indicate loss of structural content, as evident in Figures~\ref{fig:cmp-real1}--\ref{fig:cmp-real3}. Runtime $(\Lin,\Lout)$ controls (Sections~\ref{sec:exp-target-L} and~\ref{sec:exp-smart-start}) make $\gamma$-Bridge the first controllable-look-number SAR despeckler, yielding the strongest reported PSNR at $\Lobs \le 4$ on the McMaster and Kodak24 benchmarks despite the acknowledged airborne-SAR domain gap.
	
	\begin{table}[t]
		\centering
		\scriptsize
		\caption{Ratio-image statistics of $r = x_{\mathrm{obs}}/x_{t^\star}$ vs.\ target look $L_{\mathrm{out}}$ on 64 BSDS500 crops. Rightmost column: stochastic Beta-coupled theoretical reference $(L_{\mathrm{out}}{-}1)/(L_{\mathrm{out}}{+}1)$ (see Section~\ref{sec:exp-target-L} for interpretation of the small-$L_{\mathrm{out}}$ under-shoot).}
		\label{tab:target-L}
		\begin{tabular}{@{}lccc@{}}
			\toprule
			$\Lout$ & ratio mean & ratio var (emp.) & $(\Lout{-}1)/(\Lout{+}1)$ (theo.) \\
			\midrule
			2       & 0.779 & 0.195 & 0.333 \\
			4       & 0.825 & 0.368 & 0.600 \\
			16      & 0.925 & 0.706 & 0.882 \\
			66      & 0.968 & 0.895 & 0.970 \\
			266     & 0.986 & 0.971 & 0.993 \\
			10000   & 1.001 & 1.022 & 1.000 \\
			\bottomrule
		\end{tabular}
	\end{table}
	
	\begin{figure*}[t]
		\centering
		\includegraphics[width=0.8\textwidth]{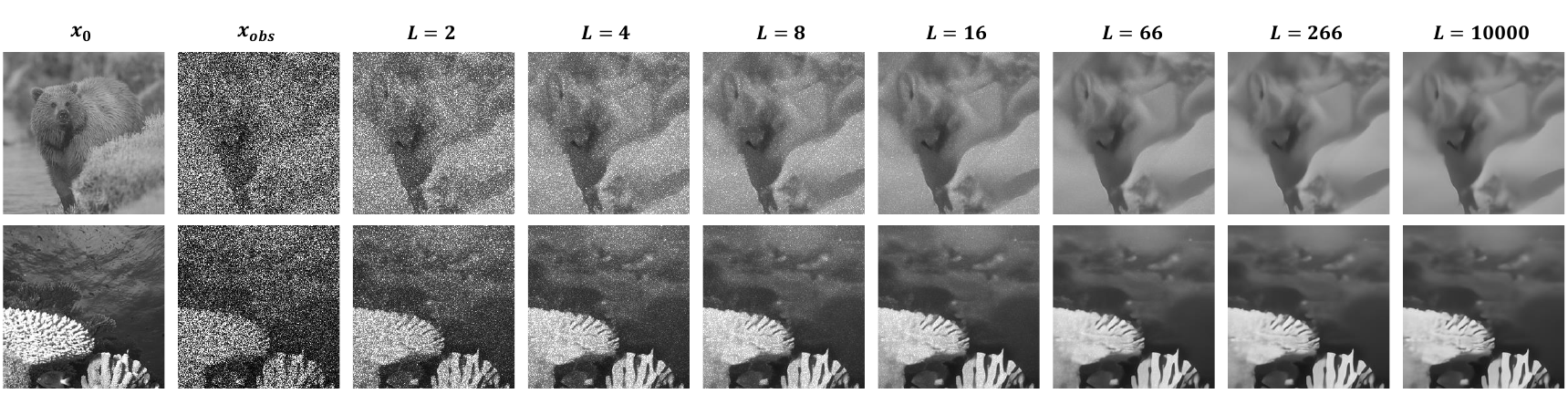}
		\caption{Target-$L$ output control from a single model.
			From left: reference $x_0$, observation, and outputs at
			$\Lout \in \{2, 4, 8, 16, 66, 266, 10000\}$.}
		\label{fig:target-L}
	\end{figure*}
	
	\subsection{Multi-step reverse at $\Lobs=1$}
	\label{sec:exp-multistep}
	
	Table~\ref{tab:multistep-l1} reports PSNR at NFE $\in \{1, 5, 25\}$ on 64 BSDS500 test crops at $\Lobs = 1$. Five reverse steps yield $22.56$\,dB, a $+1.16$\,dB gain over the one-step baseline, indicating that a capacity-limited $\Lobs{=}1$ predictor benefits from decomposing restoration into multiple stages of the Gamma--L\'evy posterior. At 25 steps, PSNR drops by $0.51$\,dB relative to 5 steps and the sample-mean ratio falls to $0.947$, revealing a systematic $\approx\!5\%$ over-shrinkage; we attribute this to accumulated train--inference joint mismatch in the $(x_t, x_{\text{obs}})$ conditioning distribution. Figure~\ref{fig:multistep} visualizes the trade-off across NFE.

	\begin{figure}[t]
		\centering
		\includegraphics[width=0.85\linewidth]{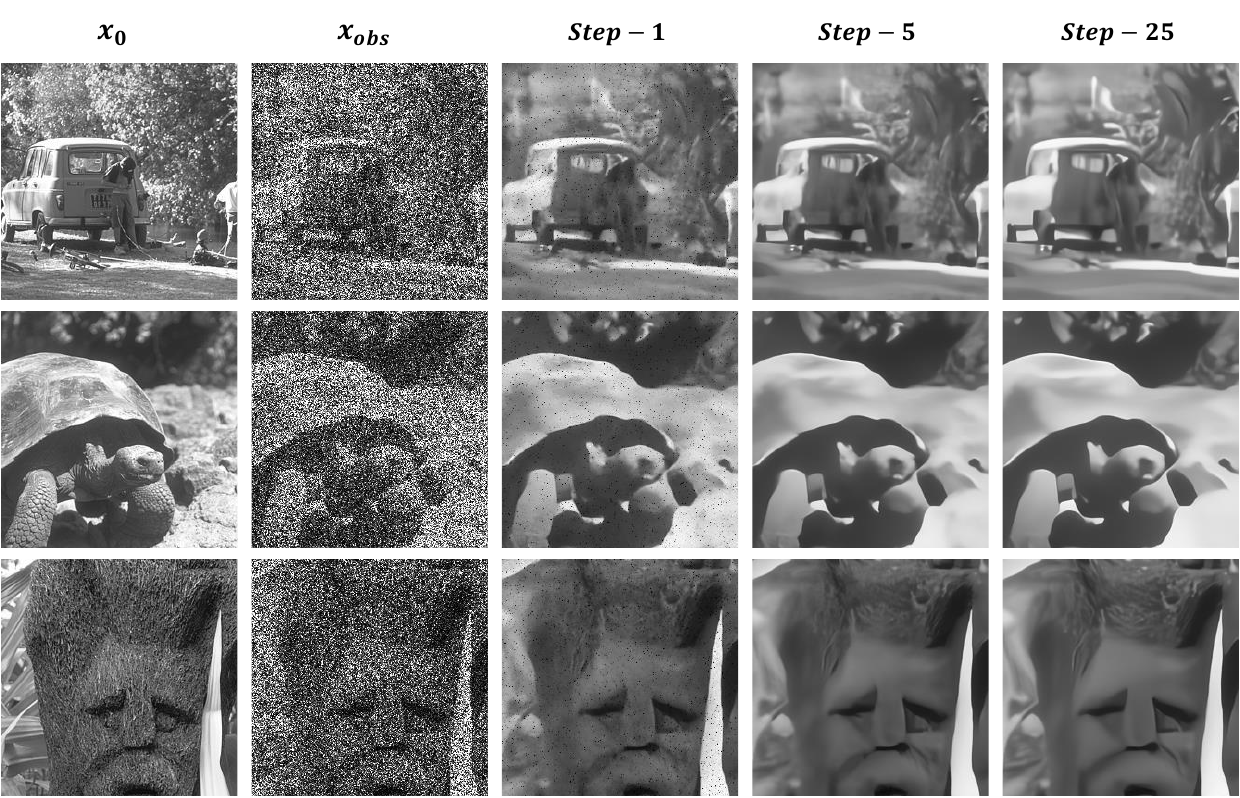}
		\caption{Multi-step reverse at $\Lobs{=}1$ on three BSDS500 crops.
			Left to right: clean $x_0$, noisy $\xobs$, and outputs at NFE $\in
			\{1, 5, 25\}$ (deterministic).}
		\label{fig:multistep}
	\end{figure}

	\subsection{Target-$L$ output control}
	\label{sec:exp-target-L}
	
	From one $\Lobs{=}1$ input, stopping the reverse chain at $\Lout \in \{2, 4, 8, 16, 66, 266, 10000\}$ produces the outputs in Fig.~\ref{fig:target-L}, showing monotonic speckle reduction across the target look range. Table~\ref{tab:target-L} reports mean and variance of $r = \xobs / x_{t^\star}$. Because $\xobs$ and $x_{t^\star}$ lie on the same trajectory, $r$ is coupled rather than a residual: under the stochastic reverse (Eq.~\eqref{eq:posterior-stoch}) with $\hat{x}_0{=}x_0$ and Prop.~\ref{prop:levy}, $r \sim (\Lout/\Lobs)\,\mathrm{Beta}(\Lobs, \Lout{-}\Lobs)$ has variance $(\Lout{-}1)/(\Lout{+}1)$ at $\Lobs{=}1$, differing from the naive forward-marginal $1/\Lout$ that assumes independence. Empirical variance approaches this reference monotonically; the under-shoot at small $\Lout$ reflects the deterministic reverse preserving only the mean (Thm.~\ref{thm:marginal}) and mild non-oracle predictor smoothing.

	\subsection{Smart-start input control}
	\label{sec:exp-smart-start}
	
	Clean-reference synthetic inputs at $\Lin \in \{1,2,4,8,16,32\}$ test the same $\Lobs{=}1$-trained model. Naive-start injects the $\Lin$-look input at $t{=}T{-}1$, \textit{i.e.}\ treats it as if it were a single-look observation; smart-start instead sets $x_{t^\star} \leftarrow y$ with $L(t^\star) \approx \Lin$, anchoring the input at its matching bridge state.
	
	At $\Lin{=}4$ the two strategies yield $24.97$ and $12.44$~dB on the same test crops; the $12.5$~dB gap quantifies the cost of treating an out-of-distribution input as $\Lobs$-look. Smart-start also gains $+2.42$~dB over the $\Lin{=}1$ baseline, consistent with the additional signal content in a 4-look input, showing that smart-start uses this content rather than discarding it.
	
	Across the six $\Lin$ levels of Table~\ref{tab:smart-start}, PSNR rises from $22.55$~dB (single-look) to $29.31$~dB (32-look) and SSIM from $0.554$ to $0.850$, all from a single training run without per-$\Lin$ fine-tuning. The ratio-image variance matches the theoretical forward-marginal $1/\Lin$ closely (rightmost columns of Table~\ref{tab:smart-start}, empirical $1.002 \to 0.028$ against theory $1.000 \to 0.031$); note that this reference differs from the Beta-coupled one used in Section~\ref{sec:exp-target-L} because $y$ is an external input independent of the internal reverse chain rather than lying on the same trajectory. Fig.~\ref{fig:cmp-synth} illustrates the qualitative progression: mild residual speckle at $\Lin{=}1$, recovered mid-scale texture at $\Lin{=}4$, and fine high-frequency detail at $\Lin{=}16$. Because higher-$\Lin$ inputs contain strictly more information about $x_0$, smart-start resolves rather than hallucinates this detail; the ratio-statistic check confirms radiometric fidelity throughout.

	\subsection{Ablation studies}
	\label{sec:exp-ablation}
	
	\begin{table}[t]
		\centering
		\scriptsize
		\caption{Smart-start input control at $\Lout\!\to\!\Lmax$: PSNR /
			SSIM / ratio-variance vs.\ input look $\Lin$ (64 crops, NFE=5
			deterministic, ckpt\_15k EMA).}
		\label{tab:smart-start}
		\begin{tabular}{@{}lcccccc@{}}
			\toprule
			$\Lin$ & step & PSNR in & PSNR out & SSIM out & ratio var & $1/\Lin$ \\
			\midrule
			1  & 99 &  9.78 & 22.55 & 0.554 & 1.002 & 1.000 \\
			2  & 92 & 11.57 & 23.69 & 0.613 & 0.484 & 0.500 \\
			4  & 84 & 13.69 & 24.97 & 0.679 & 0.241 & 0.250 \\
			8  & 77 & 16.10 & 26.28 & 0.739 & 0.120 & 0.125 \\
			16 & 69 & 18.74 & 27.76 & 0.800 & 0.058 & 0.063 \\
			32 & 62 & 21.52 & 29.31 & 0.850 & 0.028 & 0.031 \\
			\bottomrule
		\end{tabular}
	\end{table}

	On 64 BSDS500 crops, Table~\ref{tab:ablation} ablates four design components: the closed-form Gamma--L\'evy reverse posterior, joint observation-conditioning and two-step consistency, the log-residual output head, and the log-linear $L(t)$ schedule, with NFE $\in \{1,5,25\}$ deterministic reverse.
	
	\smallskip
	\noindent\textbf{Closed-form posterior.}
	Replacing the Gamma--L\'evy reverse of Eq.~\eqref{eq:posterior-stoch} with naive direct $\hat{x}_0$ propagation (ignoring the current state $x_t$) reduces NFE=5 PSNR by $1.4$~dB ($21.13$ vs.\ $22.56$) and the multi-step gain (NFE=25 vs.\ NFE=1) from $+0.65$ to $+0.47$~dB, because later reverse steps then lack information about where they are on the bridge.
	
	\smallskip
	\noindent\textbf{Observation conditioning and consistency.}
	Without $\xobs$ conditioning and $\mathcal{L}_{\mathrm{cons}}$, single-step L1-only training gives a comparable NFE=1 PSNR ($20.08$~dB) but NFE=25 PSNR collapses to $15.78$~dB (a $-4.30$~dB gap). The two components together anchor the reverse trajectory to the observation and enforce cross-step prediction consistency, keeping deep multi-step reverse chains on-manifold.
	
	\smallskip
	\noindent\textbf{Log-residual parameterisation.}
	Replacing the log-residual head $\hat{x}_0 = x_t \cdot \exp(G_\theta(\cdot))$ with a linear-residual head $\hat{x}_0 = \mathrm{ReLU}(x_t + G_\theta(\cdot))$ reduces NFE=5 PSNR to $6.47$~dB, a $16.1$~dB drop. The log-domain parameterisation converts multiplicative Gamma noise into an additive residual, matching the analytical form $\log(x_0/x_t) = -\log(G_t/L(t))$ (Remark~\ref{rem:log-residual}); the linear counterpart is a substantially harder optimisation target and fails to converge.
	
	\smallskip
	\noindent\textbf{Look-number schedule.}
	A linear $L(t)$ schedule lowers NFE=5 PSNR by $4.5$~dB. Because the dynamic range spans four orders of magnitude ($\Lobs{=}1$ to $\Lmax{=}10^{4}$), a linear schedule assigns almost the entire variance reduction to a few early steps, exceeding what a single network step can plausibly represent. The log-linear schedule distributes multiplicative variance reduction geometrically across steps, matching the natural evolution of Gamma marginals.
	
	\smallskip
	\noindent\textbf{Deterministic vs.\ stochastic reverse.} On the same $\Lobs{=}1$ checkpoint over 64 BSDS500 crops (Table~\ref{tab:otode-vs-stoch}), the deterministic reverse leads on PSNR at every NFE but its ratio moments drift at large NFE, whereas the stochastic reverse trades ${\sim}0.5$~dB for $\widehat{\mathrm{Var}}(r)$ within $8\%$ of $1$ throughout---matching Thm.~\ref{thm:marginal} (deterministic preserves the mean; stochastic preserves the full marginal).
	
	\begin{table}[t]
		\centering
		\scriptsize
		\caption{Deterministic vs.\ stochastic reverse at $\Lobs{=}1$.}
		\label{tab:otode-vs-stoch}
		\begin{tabular}{@{}llccc@{}}
			\toprule
			NFE & Reverse & PSNR & $\bar r$ & $\widehat{\mathrm{Var}}(r)$ \\
			\midrule
			\multirow{2}{*}{5}
			& determ. & \textbf{21.21} & \textbf{0.998} & \textbf{1.002} \\
			& stoch.  & 20.77 & 1.016 & 1.053 \\
			\multirow{2}{*}{25}
			& determ. & \textbf{20.71} & 0.947 & 0.907 \\
			& stoch.  & 19.96 & \textbf{1.025} & \textbf{1.075} \\
			\bottomrule
		\end{tabular}
	\end{table}
	
	\begin{figure}[t]
		\centering
		\includegraphics[width=\linewidth]{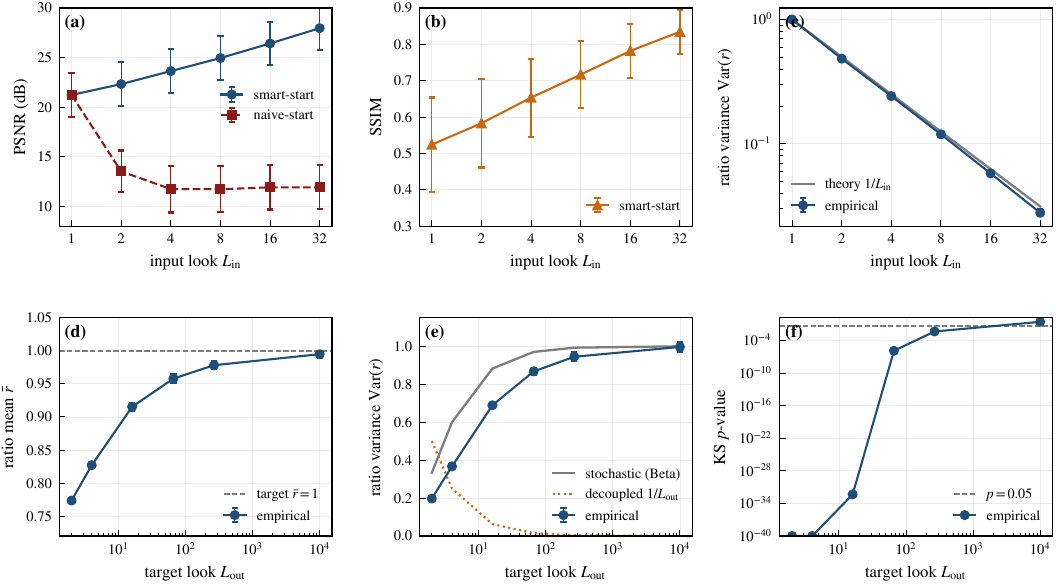}
		\caption{Runtime look-number control (error bars: $\pm 1\sigma$; ckpt\_15k EMA, NFE${=}5$ deterministic). \textbf{Top row}, smart-start along input axis $L_{\mathrm{in}}$: (a) PSNR of smart-start vs.\ naive-start; (b) SSIM; (c) empirical ratio variance vs.\ the theoretical $1/L_{\mathrm{in}}$. \textbf{Bottom row}, target-$L$ along output axis $L_{\mathrm{out}}$: (d) ratio mean $\bar{r} = x_{\mathrm{obs}}/x_{t^\star}$; (e) ratio variance against the stochastic Beta-coupled reference $(L_{\mathrm{out}}{-}1)/(L_{\mathrm{out}}{+}1)$ and the decoupled forward-marginal $1/L_{\mathrm{out}}$; (f) KS $p$-value of $r$ against $\Gamma(1,1)$, the $L_{\mathrm{out}} \to \infty$ limit of the Beta reference---values at intermediate $L_{\mathrm{out}}$ should be read as trend only.}
		\label{fig:lookcontrol-combo}
	\end{figure}

	\section{Discussion and Limitations}
	\label{sec:discussion}
	
	\noindent\textbf{Bayes limit at $\Lobs{=}1$.} Near-$0$~dB per-pixel SNR makes $p(x_0 \mid \xobs)$ intrinsically high-variance; any pointwise estimator must over-smooth. Smart-start recovers fine detail once the input carries enough signal ($\Lin \ge 4$, Table~\ref{tab:smart-start}). This is a limit of the observation model, not the algorithm.
	
	\smallskip
	\noindent\textbf{Train--inference mismatch.} During training $(x_t, c{=}\xobs)$ are independent given $x_0$, but at inference $x_t$ is generated by iterating the reverse chain from $c$, and smart-start additionally shifts the marginal of $c$ from $\Lobs$ to $\Lin$. The resulting joint deviation accumulates over multi-step reverse and manifests as the $\bar r{=}0.947$ shrinkage and $-0.51$~dB regression at NFE${=}25$ (Table~\ref{tab:multistep-l1}). Randomizing $\Lobs$ during training or raising $\lambda_{\mathrm{cons}}$ after $\mathcal{L}_{\mathrm{rec}}$ saturates are natural remedies, left to future work.
	
	
	\smallskip
	\noindent\textbf{Beyond SAR.} The construction generalizes to any scale-parameter exponential-family observation with a physical index analogous to $L$: ultrasound envelopes, photon-limited imaging, and low-light microscopy. Extending to Poisson requires only replacing Lemma~\ref{lem:gamma-add} with Poisson additivity.

	\begin{table}[t]
		\centering
		\scriptsize
		\caption{Component ablation at $\Lobs{=}1$: PSNR (dB) at NFE $\in
			\{1,5,25\}$; Gap $=$ NFE25$-$NFE1 (negative $=$ diverging).}
		\label{tab:ablation}
		\begin{tabular}{@{}lcccr@{}}
			\toprule
			Variant & NFE=1 & NFE=5 & NFE=25 & Gap \\
			\midrule
			Full model                                  & \textbf{21.40} & \textbf{22.56} & \textbf{22.05} & $\mathbf{+0.65}$ \\
			\midrule
			$-$ closed-form posterior (naive)           & 20.03 & 21.13 & 20.50 & $+0.47$ \\
			$-$ cond$_{x_1}$ $-$ consistency           & 20.08 & 20.05 & 15.78 & $-4.30$ \\
			$-$ log schedule (linear $L(t)$)            & 17.89 & 18.02 & 18.30 & $+0.41$ \\
			$-$ log-residual (linear residual)          & \phantom{0}5.15 & \phantom{0}6.47 & \phantom{0}6.67 & $+1.52$ \\
			\bottomrule
		\end{tabular}
	\end{table}
	
	\section{Conclusion}
	\label{sec:conclusion}
	
	We introduced $\gamma$-Bridge, a look-parametric diffusion bridge for multiplicative Gamma image restoration, with SAR despeckling as its primary validation setting. Anchored on an exact Gamma--L\'evy kernel and indexed by a look-number schedule, the model exposes the noise strength $L$ as a runtime-controllable dimension in both the input and output directions. A single training at $\Lobs{=}1$ supports zero-shot restoration over the full admissible $(\Lin, \Lout)$ grid---a capability absent from both prior Gamma denoisers and prior SAR-specific despecklers. The core theoretical contribution---a closed-form Gamma--L\'evy reverse posterior with two-step consistency training---enables stable multi-step inference in the low-SNR regime and improves reconstruction quality over one-step baselines. On real SAR, combining the trained model with a homogeneous-patch $\hat L$ estimator yields zero-shot transfer across six spaceborne and airborne sensors without fine-tuning.

	\appendices
	
	\section{Proof of Lemma~\ref{lem:gamma-add} (Gamma Additivity)}
	\label{app:gamma-add}
	
	Let $G_j \sim \Gam(\alpha_j, \beta)$ be independent for $j\in\{1,2\}$
	with common rate $\beta > 0$. The characteristic function of
	$\Gam(\alpha,\beta)$ is
	\begin{equation}
		\varphi_G(s) \;=\; \EE\bigl[e^{isG}\bigr] \;=\;
		\bigl(1 - is/\beta\bigr)^{-\alpha},
		\quad s\in\RR.
	\end{equation}
	By independence, $\varphi_{G_1+G_2}(s) = \varphi_{G_1}(s)\varphi_{G_2}(s)
	= (1 - is/\beta)^{-(\alpha_1+\alpha_2)}$, which is the characteristic
	function of $\Gam(\alpha_1 + \alpha_2, \beta)$. Since the
	function determines the law uniquely, $G_1 + G_2 \sim
	\Gam(\alpha_1 + \alpha_2, \beta)$. \hfill$\square$
	
	Note the critical role of a common rate: if $\beta_1 \neq \beta_2$
	the sum is not Gamma in general. This is why the L\'evy
	decomposition of Section~\ref{sec:method-levy} matches rates to $1$
	before summing, and only afterwards rescales via
	Lemma~\ref{lem:gamma-scale} to obtain the target form.

	\section{Proof of Theorem~\ref{thm:marginal} (Marginal Preservation)}
	\label{app:posterior}
	Assume the oracle predictor $\hat{x}_0 = x_0$. We prove the stochastic form first, then the deterministic form as a corollary.
	
	\smallskip
	\noindent\textbf{Stochastic reverse.}
	By hypothesis $x_t = (x_0/L(t))\,G_t$ with $G_t \sim \Gam(L(t), 1)$ (the forward marginal in Eq.~\eqref{eq:forward-def}), and by Eq.~\eqref{eq:posterior-stoch} with $\hat{x}_0 = x_0$,
	\begin{align}
		x_{t-1}
		\;&=\; \alpha_t x_t + \frac{x_0}{L(t{-}1)}\, Y_{t\to t{-}1} \nonumber\\
		\;&=\; \frac{L(t)}{L(t{-}1)} \cdot \frac{x_0}{L(t)}\,G_t + \frac{x_0}{L(t{-}1)}\, Y_{t\to t{-}1} \nonumber\\
		\;&=\; \frac{x_0}{L(t{-}1)}\,\bigl(G_t + Y_{t\to t{-}1}\bigr).
		\label{eq:app-marg-1}
	\end{align}
	By construction $G_t \sim \Gam(L(t), 1)$ and $Y_{t\to t{-}1} \sim \Gam(L(t{-}1) - L(t), 1)$ are independent with common rate $1$. Lemma~\ref{lem:gamma-add} then gives
	\begin{equation}
		G_t + Y_{t\to t{-}1} \;\sim\; \Gam\bigl(L(t{-}1),\, 1\bigr).
		\label{eq:app-marg-sum}
	\end{equation}
	Substituting Eq.~\eqref{eq:app-marg-sum} back into Eq.~\eqref{eq:app-marg-1} and applying Lemma~\ref{lem:gamma-scale} with $c = x_0 / L(t{-}1)$,
	\begin{equation}
		\begin{aligned}
			x_{t-1} \;&\sim\; \frac{x_0}{L(t{-}1)}\, \Gam\bigl(L(t{-}1),\, 1\bigr) \\
			\;&=\; x_0 \cdot \Gam\bigl(L(t{-}1),\, L(t{-}1)\bigr),
		\end{aligned}
		\label{eq:app-marg-final}
	\end{equation}
	which is precisely the forward marginal of Prop.~\ref{prop:forward-marg} at step $t{-}1$.
	
	\smallskip
	\noindent\textbf{Deterministic reverse.}
	Taking the expectation of Eq.~\eqref{eq:posterior-stoch} conditional on $(x_t, x_0)$ over the independent increment $Y_{t \to t{-}1}$, and using $\EE[Y_{t \to t{-}1}] = L(t{-}1) - L(t)$,
	\begin{equation}
		\begin{aligned}
			\EE[x_{t-1} \mid x_t, x_0]
			\;&=\; \alpha_t x_t + \frac{x_0}{L(t{-}1)}\bigl(L(t{-}1) - L(t)\bigr) \\
			\;&=\; \alpha_t x_t + (1 - \alpha_t) x_0,
		\end{aligned}
	\end{equation}
	which is exactly the deterministic update of Eq.~\eqref{eq:posterior-otode}. Taking a further expectation over $x_t$ under the forward marginal $x_t \sim x_0 \cdot \Gam(L(t), L(t))$ with $\EE[x_t] = x_0$,
	\begin{equation}
		\EE[x_{t-1}] \;=\; \alpha_t x_0 + (1 - \alpha_t) x_0 \;=\; x_0.
	\end{equation}
	Hence the deterministic reverse tracks the mean of the target marginal at every step; it does not preserve the full distributional shape, since the variance shrinks under repeated deterministic averaging---consistent with its role as a mean-tracking sampler. \qed
	
	\smallskip
	\noindent\textbf{Corollary: iterated trajectory.} By induction on $t$, the stochastic reverse chain $x_{T-1}, x_{T-2}, \ldots, x_0$ initialised at $x_{T-1} = \xobs$ under the oracle predictor visits $x_t \sim x_0 \cdot \Gam(L(t), L(t))$ at every $t$. The trajectory therefore threads through the forward marginals exactly, with $x_0$ recovered in the limit $L(t) \to \Lmax \Rightarrow \mathrm{Var}(x_t) \to 0$.

	\section{Effect of the Consistency Loss on the Gradient}
	\label{app:cons-gradient}
	
	We show that $\mathcal{L}_{\mathrm{cons}}$ (Eq.~\eqref{eq:loss-cons}) provides a training signal that $\mathcal{L}_{\mathrm{rec}}$ cannot.
	
	Consider a network $G_\theta$ with sufficient capacity to reach the pointwise $\ell_1$ minimum, i.e.\ $\hat{x}_0(x_t; t, \xobs) = \mathrm{Med}[x_0 \mid x_t, \xobs]$ at every $(x_t, t)$ visited by the forward process. At this optimum, $\nabla_\theta \mathcal{L}_{\mathrm{rec}} = 0$ and $\mathcal{L}_{\mathrm{rec}}$ can no longer guide training.
	
	Under Eq.~\eqref{eq:posterior-otode}, the $K$-step reverse composes this estimator across the schedule:
	\begin{equation}
		\begin{aligned}
			\tilde{x}_0^{(K)}(x_t) \;&=\; \underbrace{f_{K} \circ \cdots \circ f_1}_{K \text{ steps}}(x_t), \\
			f_k(x) \;&:=\; \alpha_k\, x + (1 - \alpha_k)\, \hat{x}_0(x; t_k, \xobs),
		\end{aligned}
	\end{equation}
	where $\alpha_k = L(t_k)/L(t_{k-1})$ and $t_k$ is the $k$-th step of the reverse chain. In general $\tilde{x}_0^{(K)}(x_t) \neq \hat{x}_0(x_t, \ldots)$: pointwise-median training makes each $\hat{x}_0$ prediction individually optimal, but does not constrain how these compose along a trajectory.
	
	$\mathcal{L}_{\mathrm{cons}}$ closes this gap. It probes $\hat{x}_0$ at a step $t^{\prime}$ whose input $x_{t^{\prime}}$ (Eq.~\eqref{eq:cons-x-mid}) is a deterministic function of the network's own output at $t$. Reducing $\mathcal{L}_{\mathrm{cons}}$ therefore requires $\hat{x}_0(x_{t^{\prime}}, \ldots)$ to reconstruct $x_0$ along the deterministic trajectory, not just at independently sampled points---a joint-level constraint that $\mathcal{L}_{\mathrm{rec}}$ cannot enforce. Empirically, at NFE${=}25$ vs.\ NFE${=}1$ this signal is the difference between the full model's $+0.65$~dB and the $-4.30$~dB collapse when conditioning and consistency are removed (Table~\ref{tab:ablation}).
	
	\section{Details of the Randomised-$\Lobs$ Training Scheme}
	\label{app:random-Lobs}
	
	The main results train at a fixed $\Lobs{=}1$, which induces two sources of train--inference distribution mismatch (Section~\ref{sec:discussion}): the joint $(x_t, c{=}\xobs)$ is independent during training but correlated at inference, and the marginal of $c$ shifts from $\Lobs$ to $\Lin$ whenever smart-start is used with $\Lin \neq \Lobs$. A principled fix is to randomise $\Lobs$ across mini-batches so that the conditioning channel visits the full range of look numbers that smart-start selects at inference.
	
	\smallskip
	\noindent\textbf{Training procedure.}
	At each iteration, draw $\Lobs \sim \pi$ from a distribution supported on $[1, \Lmax]$, set $L(T{-}1) = \Lobs$ (rebuilding the log-linear schedule on the fly), synthesise $\xobs = x_0 \cdot G_{\Lobs}$ with $G_{\Lobs} \sim \Gam(\Lobs, \Lobs)$, and run the usual training step. Because the network conditioning already includes $\log L(t)$, no architectural change is required.
	
	\smallskip
	\noindent\textbf{Design choices.}
	A natural choice for $\pi$ is log-uniform on $[1, 32]$: it matches both the smart-start operating range and the geometric character of the $L(t)$ schedule itself, so training coverage is uniform on the scale that the reverse chain actually traverses. An alternative is a mixture peaked at $\Lobs{=}1$ with a lighter tail toward larger $\Lobs$, which preserves training capacity at the hardest single-look setting. Because $\Lmax$ remains fixed, the schedule range $[\Lobs, \Lmax]$ is longer for small $\Lobs$; two options keep the per-batch step count comparable: a per-batch $T$ scaled with $\log(\Lmax/\Lobs)$, or a per-batch renormalisation $t \mapsto t \cdot T_{\mathrm{ref}} / T_{\Lobs}$ that maps the schedule back to a shared step index.
	
	\bibliographystyle{IEEEtran}
	\bibliography{reference}

@InProceedings{liu2023,
  title = 	 {{I}$^2${SB}: Image-to-Image Schrödinger Bridge},
  author =       {Liu, Guan-Horng and Vahdat, Arash and Huang, De-An and Theodorou, Evangelos and Nie, Weili and Anandkumar, Anima},
  booktitle = 	 {Proceedings of the 40th International Conference on Machine Learning},
  pages = 	 {22042--22062},
  year = 	 {2023},
  editor = 	 {Krause, Andreas and Brunskill, Emma and Cho, Kyunghyun and Engelhardt, Barbara and Sabato, Sivan and Scarlett, Jonathan},
  volume = 	 {202},
  series = 	 {Proceedings of Machine Learning Research},
  month = 	 {23--29 Jul},
  publisher =    {PMLR},
}

@inproceedings{ICLR2024_20e45668,
 author = {Zhou, Linqi and Lou, Aaron and Khanna, Samar and Ermon, Stefano},
 booktitle = {International Conference on Learning Representations},
 editor = {B. Kim and Y. Yue and S. Chaudhuri and K. Fragkiadaki and M. Khan and Y. Sun},
 pages = {8160--8171},
 title = {Denoising Diffusion Bridge Models},
 volume = {2024},
 year = {2024}
}

@inproceedings{ICLR2024_54912807,
 author = {Kim, Beomsu and Kwon, Gihyun and Kim, Kwanyoung and YE, Jong Chul},
 booktitle = {International Conference on Learning Representations},
 editor = {B. Kim and Y. Yue and S. Chaudhuri and K. Fragkiadaki and M. Khan and Y. Sun},
 pages = {19312--19331},
 title = {Unpaired Image-to-Image Translation via Neural Schr\"{o}dinger Bridge},
 volume = {2024},
 year = {2024}
}

@article{nachmani2021denoising,
	title={Denoising diffusion gamma models},
	author={Nachmani, Eliya and Roman, Robin San and Wolf, Lior},
	journal={arXiv preprint arXiv:2110.05948},
	year={2021}
}

@article{xie2023diffusion,
	title={Diffusion model for generative image denoising},
	author={Xie, Yutong and Yuan, Minne and Dong, Bin and Li, Quanzheng},
	journal={arXiv preprint arXiv:2302.02398},
	year={2023}
}

@ARTICLE{10109106,
  author={Perera, Malsha V. and Nair, Nithin Gopalakrishnan and Bandara, Wele Gedara Chaminda and Patel, Vishal M.},
  journal={IEEE Geoscience and Remote Sensing Letters}, 
  title={SAR Despeckling Using a Denoising Diffusion Probabilistic Model}, 
  year={2023},
  volume={20},
  number={},
  pages={1-5},
  doi={10.1109/LGRS.2023.3270799}}

@article{guha2023sddpm,
	title={Sddpm: Speckle denoising diffusion probabilistic models},
	author={Guha, Soumee and Acton, Scott T},
	journal={arXiv preprint arXiv:2311.10868},
	year={2023}
}

@article{heo2026self,
	title={Self-Supervised Score-Based Despeckling for SAR Imagery via Log-Domain Transformation},
	author={Heo, Junhyuk},
	journal={arXiv preprint arXiv:2601.14334},
	year={2026}
}

@ARTICLE{sar2sar,
  author={Dalsasso, Emanuele and Denis, Loïc and Tupin, Florence},
  journal={IEEE Journal of Selected Topics in Applied Earth Observations and Remote Sensing}, 
  title={SAR2SAR: A Semi-Supervised Despeckling Algorithm for SAR Images}, 
  year={2021},
  volume={14},
  number={},
  pages={4321-4329},
  doi={10.1109/JSTARS.2021.3071864}
  }

@article{dalsasso2021if,
	title={As if by magic: Self-supervised training of deep despeckling networks with MERLIN},
	author={Dalsasso, Emanuele and Denis, Lo{\"\i}c and Tupin, Florence},
	journal={IEEE Transactions on Geoscience and Remote Sensing},
	volume={60},
	pages={1--13},
	year={2021},
	publisher={IEEE}
}

@article{molini2021speckle2void,
	title={Speckle2Void: Deep self-supervised SAR despeckling with blind-spot convolutional neural networks},
	author={Molini, Andrea Bordone and Valsesia, Diego and Fracastoro, Giulia and Magli, Enrico},
	journal={IEEE Transactions on Geoscience and Remote Sensing},
	volume={60},
	pages={1--17},
	year={2021},
	publisher={IEEE}
}

@ARTICLE{Lee1980,
  author={Lee, Jong-Sen},
  journal={IEEE Transactions on Pattern Analysis and Machine Intelligence}, 
  title={Digital Image Enhancement and Noise Filtering by Use of Local Statistics}, 
  year={1980},
  volume={PAMI-2},
  number={2},
  pages={165-168},
  doi={10.1109/TPAMI.1980.4766994}
  }

@ARTICLE{Frost1982,
  author={Frost, Victor S. and Stiles, Josephine Abbott and Shanmugan, K. S. and Holtzman, Julian C.},
  journal={IEEE Transactions on Pattern Analysis and Machine Intelligence}, 
  title={A Model for Radar Images and Its Application to Adaptive Digital Filtering of Multiplicative Noise}, 
  year={1982},
  volume={PAMI-4},
  number={2},
  pages={157-166},
  doi={10.1109/TPAMI.1982.4767223}}

@ARTICLE{Parrilli2012,
  author={Parrilli, Sara and Poderico, Mariana and Angelino, Cesario Vincenzo and Verdoliva, Luisa},
  journal={IEEE Transactions on Geoscience and Remote Sensing}, 
  title={A Nonlocal SAR Image Denoising Algorithm Based on LLMMSE Wavelet Shrinkage}, 
  year={2012},
  volume={50},
  number={2},
  pages={606-616},
  doi={10.1109/TGRS.2011.2161586}}

@ARTICLE{Deledalle2015,
  author={Deledalle, Charles-Alban and Denis, Loïc and Tupin, Florence and Reigber, Andreas and Jäger, Marc},
  journal={IEEE Transactions on Geoscience and Remote Sensing}, 
  title={NL-SAR: A Unified Nonlocal Framework for Resolution-Preserving (Pol)(In)SAR Denoising}, 
  year={2015},
  volume={53},
  number={4},
  pages={2021-2038},
  doi={10.1109/TGRS.2014.2352555}}

@ARTICLE{Mulog2017,
  author={Deledalle, Charles-Alban and Denis, Loïc and Tabti, Sonia and Tupin, Florence},
  journal={IEEE Transactions on Image Processing}, 
  title={MuLoG, or How to Apply Gaussian Denoisers to Multi-Channel SAR Speckle Reduction?}, 
  year={2017},
  volume={26},
  number={9},
  pages={4389-4403},
  doi={10.1109/TIP.2017.2713946}}

@INPROCEEDINGS{sarcnn,
  author={Chierchia, G. and Cozzolino, D. and Poggi, G. and Verdoliva, L.},
  booktitle={2017 IEEE International Geoscience and Remote Sensing Symposium (IGARSS)}, 
  title={SAR image despeckling through convolutional neural networks}, 
  year={2017},
  volume={},
  number={},
  pages={5438-5441},
  doi={10.1109/IGARSS.2017.8128234}}

@ARTICLE{idcnn,
  author={Wang, Puyang and Zhang, He and Patel, Vishal M.},
  journal={IEEE Signal Processing Letters}, 
  title={SAR Image Despeckling Using a Convolutional Neural Network}, 
  year={2017},
  volume={24},
  number={12},
  pages={1763-1767},
  doi={10.1109/LSP.2017.2758203}}

@article{hu2024sar,
  title={Sar despeckling via log-yeo-johnson transformation and sparse representation},
  author={Hu, Xuran and Zhu, Mingzhe and Stankovi{\'c}, Djordje and Feng, Zhenpeng and Mao, Shouhan and Stankovi{\'c}, Ljubi{\v{s}}a},
  journal={arXiv preprint arXiv:2412.18121},
  year={2024}
}

@ARTICLE{4359947,
  author={Coupe, Pierrick and Yger, Pierre and Prima, Sylvain and Hellier, Pierre and Kervrann, Charles and Barillot, Christian},
  journal={IEEE Transactions on Medical Imaging}, 
  title={An Optimized Blockwise Nonlocal Means Denoising Filter for 3-D Magnetic Resonance Images}, 
  year={2008},
  volume={27},
  number={4},
  pages={425-441},
  doi={10.1109/TMI.2007.906087}}

@InProceedings{Ulyanov_2018_CVPR,
author = {Ulyanov, Dmitry and Vedaldi, Andrea and Lempitsky, Victor},
title = {Deep Image Prior},
booktitle = {Proceedings of the IEEE Conference on Computer Vision and Pattern Recognition (CVPR)},
month = {June},
year = {2018}
}

@ARTICLE{7839189,
  author={Zhang, Kai and Zuo, Wangmeng and Chen, Yunjin and Meng, Deyu and Zhang, Lei},
  journal={IEEE Transactions on Image Processing}, 
  title={Beyond a Gaussian Denoiser: Residual Learning of Deep CNN for Image Denoising}, 
  year={2017},
  volume={26},
  number={7},
  pages={3142-3155},
  doi={10.1109/TIP.2017.2662206}}

@article{ko2021sar,
  title={SAR image despeckling using continuous attention module},
  author={Ko, Jaekyun and Lee, Sanghwan},
  journal={IEEE Journal of Selected Topics in Applied Earth Observations and Remote Sensing},
  volume={15},
  pages={3--19},
  year={2021},
  publisher={IEEE}
}

@ARTICLE{9755131,
  author={Thakur, Ramesh Kumar and Maji, Suman Kumar},
  journal={IEEE Geoscience and Remote Sensing Letters}, 
  title={AGSDNet: Attention and Gradient-Based SAR Denoising Network}, 
  year={2022},
  volume={19},
  number={},
  pages={1-5},
  doi={10.1109/LGRS.2022.3166565}}

@INPROCEEDINGS{9883415,
  author={Thakur Ramesh Kumar and Maji Suman Kumar},
  booktitle={IGARSS 2022 - 2022 IEEE International Geoscience and Remote Sensing Symposium}, 
  title={SIFSDNet: Sharp Image Feature Based SAR Denoising Network}, 
  year={2022},
  volume={},
  number={},
  pages={3428-3431},
  doi={10.1109/IGARSS46834.2022.9883415}}

@INPROCEEDINGS{9324183,
  author={Molini, Andrea Bordone and Valsesia, Diego and Fracastoro, Giulia and Magli, Enrico},
  booktitle={IGARSS 2020 - 2020 IEEE International Geoscience and Remote Sensing Symposium}, 
  title={Towards Deep Unsupervised Sar Despeckling with Blind-Spot Convolutional Neural Networks}, 
  year={2020},
  volume={},
  number={},
  pages={2507-2510},
  doi={10.1109/IGARSS39084.2020.9324183}
  }

@INPROCEEDINGS{9884596,
  author={Perera, Malsha V. and Bandara, Wele Gedara Chaminda and Valanarasu, Jeya Maria Jose and Patel, Vishal M.},
  booktitle={IGARSS 2022 - 2022 IEEE International Geoscience and Remote Sensing Symposium}, 
  title={Transformer-Based SAR Image Despeckling}, 
  year={2022},
  volume={},
  number={},
  pages={751-754},
  doi={10.1109/IGARSS46834.2022.9884596}}

@article{FANG2024376,
title = {Contrastive learning for real SAR image despeckling},
journal = {ISPRS Journal of Photogrammetry and Remote Sensing},
volume = {218},
pages = {376-391},
year = {2024},
issn = {0924-2716},
author = {Yangtian Fang and Rui Liu and Yini Peng and Jianjun Guan and Duidui Li and Xin Tian},
}

@article{albisani2025self,
  title={Self-supervised SAR despeckling using deep image prior},
  author={Albisani, Chiara and Baracchi, Daniele and Piva, Alessandro and Argenti, Fabrizio},
  journal={Pattern Recognition Letters},
  volume={190},
  pages={169--176},
  year={2025},
  publisher={Elsevier}
}

@INPROCEEDINGS{10641283,
  author={Hu, Xuran and Xu, Ziqiang and Chen, Zhihan and Feng, Zhenpeng and Zhu, Mingzhe and Stanković, Ljubiša},
  booktitle={IGARSS 2024 - 2024 IEEE International Geoscience and Remote Sensing Symposium}, 
  title={SAR Despeckling Via Regional Denoising Diffusion Probabilistic Model}, 
  year={2024},
  volume={},
  number={},
  pages={7226-7230},
  doi={10.1109/IGARSS53475.2024.10641283}}

@ARTICLE{Jose2010,
  author={Bioucas-Dias, José M. and Figueiredo, Mário A. T.},
  journal={IEEE Transactions on Image Processing}, 
  title={Multiplicative Noise Removal Using Variable Splitting and Constrained Optimization}, 
  year={2010},
  volume={19},
  number={7},
  pages={1720-1730},
  doi={10.1109/TIP.2010.2045029}
  }

@ARTICLE{DTFAD,
  author={Wei, Jiali and Liao, Xiaofeng},
  journal={IEEE Transactions on Image Processing}, 
  title={Dynamical Threshold-Based Fractional Anisotropic Diffusion for Speckle Noise Removal}, 
  year={2025},
  volume={34},
  number={},
  pages={2826-2839},
  doi={10.1109/TIP.2025.3561685}}

@article{GAO20241,
title = {Fractional-order cross-diffusion system for multiplicative noise removal},
journal = {Computers \& Mathematics with Applications},
volume = {164},
pages = {1-11},
year = {2024},
issn = {0898-1221},
author = {Juanjuan Gao and Jiebao Sun and Shengzhu Shi},
}

@ARTICLE{Lu2025DiSpeckle,
  author={Lu, Danwei and Liu, Chao and Yin, Junjun and Yang, Jian},
  journal={IEEE Transactions on Geoscience and Remote Sensing}, 
  title={DiSpeckle: Diffusion Model That Unwinds Speckle Formation With Off-the-Shelf Gaussian Denoisers}, 
  year={2025},
  volume={63},
  number={},
  pages={1-22},
  doi={10.1109/TGRS.2025.3630133}}

@article{Ran2026Tunable,
title = {A tunable despeckling neural network stabilized via diffusion equation},
journal = {Signal Processing},
volume = {239},
pages = {110324},
year = {2026},
issn = {0165-1684},
author = {Yi Ran and Zhichang Guo and Jia Li and Yao Li and Martin Burger and Boying Wu},
}

@ARTICLE{Deledalle2009PPB,
  author={Deledalle, Charles-Alban and Denis, LoÏc and Tupin, Florence},
  journal={IEEE Transactions on Image Processing}, 
  title={Iterative Weighted Maximum Likelihood Denoising With Probabilistic Patch-Based Weights}, 
  year={2009},
  volume={18},
  number={12},
  pages={2661-2672},
  doi={10.1109/TIP.2009.2029593}}

@ARTICLE{Woo2012Speckle,
  author={Woo, Hyenkyun and Yun, Sangwoon},
  journal={IEEE Transactions on Image Processing}, 
  title={Alternating Minimization Algorithm for Speckle Reduction With a Shifting Technique}, 
  year={2012},
  volume={21},
  number={4},
  pages={1701-1714},
  doi={10.1109/TIP.2011.2176345}}

@inproceedings{Okhotin2023SSDDPM,
 author = {Okhotin, Andrey and Molchanov, Dmitry and Vladimir, Arkhipkin and Bartosh, Grigory and Ohanesian, Viktor and Alanov, Aibek and Vetrov, Dmitry},
 booktitle = {Advances in Neural Information Processing Systems},
 editor = {A. Oh and T. Naumann and A. Globerson and K. Saenko and M. Hardt and S. Levine},
 pages = {10038--10067},
 publisher = {Curran Associates, Inc.},
 title = {Star-Shaped Denoising Diffusion Probabilistic Models},
 volume = {36},
 year = {2023}
}

@misc{Kim2024CTM,
      title={Consistency Trajectory Models: Learning Probability Flow ODE Trajectory of Diffusion}, 
      author={Dongjun Kim and Chieh-Hsin Lai and Wei-Hsiang Liao and Naoki Murata and Yuhta Takida and Toshimitsu Uesaka and Yutong He and Yuki Mitsufuji and Stefano Ermon},
      year={2024},
      eprint={2310.02279},
      archivePrefix={arXiv},
      primaryClass={cs.LG},
}

@ARTICLE{Bo2025SDUD,
  author={Bo, Fuyu and Ma, Xiaole and Hu, Shaohai and An, Gaoyun and Li, Yidong and Cen, Yigang},
  journal={IEEE Journal of Selected Topics in Applied Earth Observations and Remote Sensing}, 
  title={Speckle-Driven Unsupervised Despeckling for SAR Images}, 
  year={2025},
  volume={18},
  number={},
  pages={13023-13034},
  doi={10.1109/JSTARS.2025.3568854}}

@ARTICLE{Cai2026DAPhysDiff,
  author={Cai, Xingquan and Wang, Luyao and Zhang, Yupeng and Wang, Meirui and Li, Ying},
  journal={IEEE Geoscience and Remote Sensing Letters}, 
  title={DA-PhysDiff: A Dual-Aware Physics-Informed Diffusion Model for SAR Image Despeckling}, 
  year={2026},
  volume={23},
  number={},
  pages={1-5},
  doi={10.1109/LGRS.2026.3652373}}

@inproceedings{Zheng2025DBIM,
 author = {Zheng, Kaiwen and He, Guande and Chen, Jianfei and Bao, Fan and Zhu, Jun},
 booktitle = {International Conference on Learning Representations},
 editor = {Y. Yue and A. Garg and N. Peng and F. Sha and R. Yu},
 pages = {81857--81884},
 title = {Diffusion Bridge Implicit Models},

 volume = {2025},
 year = {2025}
}

@CONFERENCE{Kim2025GCTM,
	author = {Kim, Beomsu and Kim, Jaemin and Kim, Jeongsol and Ye, Jong Chul},
	title = {GENERALIZED CONSISTENCY TRAJECTORY MODELS FOR IMAGE MANIPULATION},
	year = {2025},
	journal = {13th International Conference on Learning Representations, ICLR 2025},
	pages = {101753 – 101775},
	type = {Conference paper},
	source = {Scopus}
}

@inproceedings{wang2026residual,
	title={Residual diffusion bridge model for image restoration},
	author={Wang, Hebaixu and Zhang, Jing and Chen, Haoyang and Guo, Haonan and Wang, Di and Ma, Jiayi and Du, Bo},
	booktitle={Proceedings of the IEEE/CVF Conference on Computer Vision and Pattern Recognition},
	pages={8375--8386},
	year={2026}
}
	
	\vfill
	
\end{document}